\documentclass[twoside]{article}

\usepackage[utf8]{inputenc}
\usepackage{amsfonts}
\usepackage{amsmath,amssymb}
\usepackage{nccmath}
\usepackage{mathrsfs}
\usepackage{graphicx}
\usepackage{subfigure}
\usepackage{tabularx}
\usepackage[all]{xy}
\usepackage{color}
\usepackage{multirow}
\usepackage{booktabs}
\usepackage{enumitem}

\newtheorem{proposition}{Proposition}[section]
\newtheorem{lemma}[proposition]{Lemma}

\newtheorem{definition}[proposition]{Definition}
\newtheorem{remark}[proposition]{Remark}

\newtheorem{example}[proposition]{Example}

\newenvironment{proof}[1][Proof\ ]{\medskip\noindent{\bf #1}\ }{%
\hfill $\Box$\par\quad\par}
\newenvironment{mythm}[1][]{\medskip\par\noindent{\bf #1}\ \,\,\em}{\par}

\usepackage[round]{natbib}

\def\mcl#1{\mathcal{#1}}
\def\bracket#1{\left\langle #1\right\rangle}
\def\hil{\mcl{H}}
\def\nn{\nonumber}
\def\opn{\operatorname}
\def\mr{\mathrm}

\def\alg{\mcl{A}}
\def\modu{\mcl{M}}

\def\bbracket#1{\big\langle #1\big\rangle}
\def\Bbracket#1{\bigg\langle #1\bigg\rangle}

\def\bc{\mathbf{c}}
\def\bG{\mathbf{G}}
\def\by{\mathbf{y}}
\def\bF{\mathbf{F}}

\def\red#1{\textcolor{black}{#1}}

\DeclareSymbolFont{EulerExtension}{U}{euex}{m}{n}
\DeclareMathSymbol{\euintop}{\mathop} {EulerExtension}{"52}
\DeclareMathSymbol{\euointop}{\mathop} {EulerExtension}{"48}

\definecolor{MyDarkBlue}{rgb}{0.05,0.,0.8} 
\newcommand{\addHK}[1]{{\color{black}  #1}}
\usepackage{todonotes}

\begin{document}

%

%

\title{Learning in RKHM: a $C^*$-Algebraic Twist for Kernel Machines}

\author{Yuka Hashimoto$^{1,2}$\quad Masahiro Ikeda$^{2,3}$\quad Hachem Kadri$^4$\medskip\\
{\normalsize 1. NTT Network Service Systems Laboratories, NTT Corporation, Tokyo, Japan}\\
{\normalsize 2. Center for Advanced Intelligence Project, RIKEN, Tokyo, Japan}\\
{\normalsize 3. Keio University, Yokohama, Japan}\\
{\normalsize 4. Aix-Marseille University, CNRS, LIS, Marseille, France}}

\date{}

\maketitle

\begin{abstract}
Supervised learning in reproducing kernel Hilbert space~(RKHS) and vector-valued RKHS~(vvRKHS) has been investigated for more than 30 years. In this paper, we provide a new twist to this rich literature by generalizing supervised learning in RKHS and vvRKHS to reproducing kernel Hilbert $C^*$-module~(RKHM), and show how to construct effective positive-definite kernels by considering the perspective of $C^*$-algebra. Unlike the cases of RKHS and vvRKHS, we can use $C^*$-algebras to enlarge representation spaces. This enables us to construct RKHMs whose representation power goes beyond RKHSs, vvRKHSs, and existing methods such as convolutional neural networks. Our framework is suitable, for example,  for effectively analyzing image data by allowing the interaction of Fourier components. 
\end{abstract}

\section{INTRODUCTION}
Supervised learning in reproducing kernel Hilbert space~(RKHS) has been actively investigated \addHK{since the early 1990s ~\citep{murphy12, christmann2008support, shawe2004kernel, scholkopf2002learning, boser1992training}}.
\addHK{The notion of reproducing kernels as dot products in Hilbert spaces was first brought to the field of machine learning by~\citet{aizerman1964theoretical}, while the theoretical foundation of reproducing kernels and their Hilbert spaces dates back to at least~\citet{aronszajn1950theory}.}
By virtue of the representer theorem~\citep{scholkopf01_representer}, we can compute the solution of an {infinite-dimensional} minimization problem \addHK{in RKHS} with given finite samples.
In addition to the standard RKHSs, applying vector-valued RKHSs (vvRKHSs) to supervised learning has also been proposed and used in analyzing vector-valued data~\citep{micchelli05,alvarez2012kernels,kadri16,quang16, brouard2016input,laforgue20,huusari21}. {Generalization bounds} of the supervised problems in RKHS and vvRKHS are also derived~\citep{mohri18,caponnetto2007optimal, audiffren2013stability, huusari21}. 

Reproducing kernel Hilbert $C^*$-module (RKHM) is a generalization of RKHS and vvRKHS by means of $C^*$-algebra.
$C^*$-algebra is a generalization of the space of complex values.
It \red{has a product and an involution structures}.
Important examples are the $C^*$-algebra of bounded linear operators on a Hilbert space and the $C^*$-algebra of continuous functions on a compact space.
RKHMs have been \red{originally} studied for pure operator algebraic and mathematical physics problems~\citep{manuilov00,heo08,moslehian22}.
Recently, applying RKHMs to data analysis has been proposed by~\citet{hashimoto21}.
They generalized the representer theorem in RKHS to RKHM, which allows us to analyze structured data such as functional data with $C^*$-algebras. 

\addHK{In this paper, we investigate supervised learning in RKHM. This provides a new twist to the state-of-the-art kernel-based learning algorithms and the development of a
novel kind of reproducing kernels.}
{An advantage of RKHM over RKHS and vvRKHS is that we can enlarge the $C^*$-algebra characterizing the RKHM \red{to construct a representation space}.}
This allows us to represent more functions than the case of RKHS and make use of the product structure in the $C^*$-algebra.
Our main contributions are: 
\begin{itemize}
\item We define positive definite kernels {from the perspective of $C^*$-algebra}, which are suitable \addHK{for learning in RKHM and adapted} to analyze image data. 
\item We derive a generalization bound of the supervised \addHK{learning} problem in RKHM, which generalizes existing results of RKHS and vvRKHS.
We also show that the computational complexity of our method can be reduced if parameters in the $C^*$-algebra-valued positive definite kernels have specific structures.
\item We show that our framework generalizes existing methods \red{based on convolution operations}.
\end{itemize}

Important applications of the supervised learning in RKHM are tasks whose inputs and outputs are images.
\red{If the proposed kernels have specific parameters}, then the product structure is the convolution,
which corresponds to the pointwise product of Fourier components.
By extending the $C^*$-algebra to a larger one, we can enjoy more general operations than the convolutions.
This enables us to analyze image data effectively \red{by making interactions between Fourier components}.
Regarding the generalization bound, we derive the same type of bound as those \addHK{obtained for} RKHS and vvRKHS \addHK{via Rademacher complexity theory}. 
{This is \addHK{to our knowledge,} the first generalization bound for RKHM \addHK{hypothesis classes}.}
Concerning the connection with existing methods, we show that using our framework, we can reconstruct existing methods such as the convolutional neural network\addHK{~\citep{lecun98}} and the convolutional kernel~\citep{marial14} and further generalize them.
This fact implies that the representation power of our framework goes beyond the existing methods.

The remainder of this paper is organized as follows:
In Section~\ref{sec:background}, we review mathematical notions related to this paper.
We propose $C^*$-algebra-valued positive definite kernels in Section~\ref{sec:kernels} and investigate supervised learning in RKHM in Section~\ref{sec:supervised}.
Then, we show connections 
\addHK{with } existing \addHK{convolution-based} methods in Section~\ref{sec:existing}.
We confirm the advantage of our method numerically in Section~\ref{sec:numerical_results} and conclude the paper in Section~\ref{sec:conclusion}.
All technical proofs are in Section~\ref{ap:proofs}.

\section{PRELIMINARIES}\label{sec:background}
\subsection{$C^*$-Algebra and Hilbert $C^*$-Module}
$C^*$-algebra is a Banach space equipped with a product and \addHK{an} involution that satisfies the $C^*$ identity.
See Section~\ref{ap:c_algebra} for more details.
An example of $C^*$-algebra is group $C^*$-algebra~{\citep{kirillov76}}.
\red{Let $p\in\mathbb{N}$, and let $\mathbb{Z}/p\mathbb{Z}$ be the set of integers modulo $p$.}
\begin{definition}[Group $C^*$-algebra on a finite cyclic group]\label{def:group_c_algebra}
\red{Let $\omega=\mr{e}^{2\pi\sqrt{-1}/p}$.}
The group $C^*$-algebra on $\mathbb{Z}/p\mathbb{Z}$, which is denoted as $C^*(\mathbb{Z}/p\mathbb{Z})$, is the set of maps from $\mathbb{Z}/p\mathbb{Z}$ to $\mathbb{C}$ equipped with \red{the following} product, involution, and norm:
\begin{itemize}
\item $(x\cdot y)(z)=\sum_{w\in \mathbb{Z}/p\mathbb{Z}}x(z-w)y(w)$ \red{for $z\in\mathbb{Z}/p\mathbb{Z}$},
\item $x^*(z)=\overline{x(-z)}$,
\item $\Vert x\Vert=\max_{n\in\{0,\ldots,p-1\}}\vert\sum_{z\in \mathbb{Z}/p\mathbb{Z}}x(z)\omega^{zn}\vert$.
\end{itemize}
\end{definition}
{Since the product is the convolution, \red{group $C^*$-algebras} offer a new way to define positive definite kernels, which are effective in analyzing image data as we \addHK{will} see in Section~\ref{sec:kernels}.}
{Elements in $C^*(\mathbb{Z}/p\mathbb{Z})$ \addHK{are} described by circulant matrices~\citep{gray06}.}
Let $Circ(p)=\{x\in\mathbb{C}^{p\times p}\,\mid\, x\mbox{ is a circulant matrix}\}$.
Moreover, we denote the circulant matrix whose first row is $v$ as $\opn{circ}(v)$.
The discrete Fourier transform (DFT) matrix, whose $(i,j)$-entry is $\omega^{(i-1)(j-1)}/\sqrt{p}$, is denoted as $F$.
\begin{lemma}\label{lem:circulant_decom}
Any circulant matrix $x\in Circ(p)$ has an eigenvalue decomposition $x=F\Lambda_x F^*$, where
\begin{equation*}
\Lambda_x=\opn{diag}\bigg(\sum_{z\in\mathbb{Z}/p\mathbb{Z}}x(z)\omega^{z\cdot 0},\ldots,\sum_{z\in\mathbb{Z}/p\mathbb{Z}}x(z)\omega^{z(p-1)}\bigg).
\end{equation*}
\end{lemma}
\begin{lemma}\label{lem:circulant_isomorphic}
The group $C^*$-algebra $C^*(\mathbb{Z}/p\mathbb{Z})$ is $C^*$-isomorphic to $Circ(p)$.
\end{lemma}

We \addHK{now} review important notions about $C^*$-algebra.
We denote a $C^*$-algebra by $\alg$.
\begin{definition}[Positive]~\label{def:positive}
An element $a$ of $\alg$ is called {\em positive} if there exists $b\in\alg$ such that $a=b^*b$ holds.
For $a,b\in\alg$, we \red{write} $a\le_{\alg} b$ if $b-a$ is positive, and $a\lneq_{\alg}b$ if $b-a$ is positive and not zero.
We denote by $\alg_+$ the subset of $\alg$ composed of all positive elements in $\alg$.
\end{definition}
\begin{definition}[Minimum]\label{def:sup}
For a subset $\mcl{S}$ of $\alg$, $a\in\alg$ is said to be a {\em lower bound} with respect to the order $\le_{\alg}$, if $a\le_{\alg} b$ for any $b\in\mcl{S}$.
Then, a lower bound $c\in\alg$ is said to be an {\em infimum} of $\mcl{S}$, if $a\le_{\alg} c$ for any lower bound $a$ of $\mcl{S}$. 
If $c\in\mcl{S}$, then $c$ is said to be a {\em minimum} of $\mcl{S}$.
\end{definition}

\red{Hilbert $C^*$-module is a generalization of Hilbert space.}
We can define an $\alg$-valued inner product and a (real nonnegative-valued) norm as a natural generalization of the complex-valued inner product.
See Section~\ref{ap:c_algebra} for further details.
Then, we \red{define} Hilbert $C^*$-module as follows. 
\begin{definition}[Hilbert $C^*$-module]\label{def:hil_c*module}
Let $\modu$ be a $C^*$-module over $\alg$ equipped with an $\alg$-valued inner product.
If $\modu$ is complete with respect to the norm induced by the $\alg$-valued inner product, it is called a {\em Hilbert $C^*$-module} over $\alg$ or {\em Hilbert $\alg$-module}.
\end{definition}

\subsection{Reproducing Kernel Hilbert $C^*$-Module}

\addHK{RKHM is a generalization of RKHS by means of $C^*$-algebra.}
Let $\mcl{X}$ be a non-empty set for data.
\begin{definition}[$\alg$-valued positive definite kernel]\label{def:pdk_rkhm}
 An $\alg$-valued map $k:\mcl{X}\times \mcl{X}\to\alg$ is called a {\em positive definite kernel} if it satisfies the following conditions: 
\begin{itemize}
\item $k(x,y)=k(y,x)^*$ \;for $x,y\in\mcl{X}$,
\item $\sum_{i,j=1}^nc_i^*k(x_i,x_j)c_j\ge_{\alg} 0$ \;for $n\in\mathbb{N}$, $c_i\in\alg$, $x_i\in\mcl{X}$.
\end{itemize}
\end{definition}
\color{black}
Let $\phi: \mcl{X}\to\alg^{\mcl{X}}$ be the {\em feature map} associated with $k$, which is defined as $\phi(x)=k(\cdot,x)$ for $x\in\mcl{X}$.
We construct the following $C^*$-module composed of $\alg$-valued functions: 
\begin{equation*}
\modu_{k,0}:=\bigg\{\sum_{i=1}^{n}\phi(x_i)c_i\bigg|\ n\in\mathbb{N},\ c_i\in\alg,\ x_i\in \mcl{X}\bigg\}.
\end{equation*}
Define an $\alg$-valued map $\bracket{\cdot,\cdot}_{\modu_k}:\modu_{k,0}\times \modu_{k,0}\to\alg$ as
\begin{equation*}
\bbracket{\sum_{i=1}^{n}\phi(x_i)c_i,\sum_{j=1}^{l}\phi(y_j)b_j}_{\modu_k}:=\sum_{i=1}^{n}\sum_{j=1}^{l}c_i^*k(x_i,y_j)b_j.
\end{equation*}
By the properties of $k$ in Definition~\ref{def:pdk_rkhm}, $\bracket{\cdot,\cdot}_{\modu_k}$ is well-defined and has the reproducing property
\begin{equation*}
\bracket{\phi(x),v}_{\modu_k}=v(x),
\end{equation*}
for $v\in\modu_{k,0}$ and $x\in \mcl{X}$.
Also, it is an $\alg$-valued inner product. 
The {\em reproducing kernel Hilbert $\alg$-module~(RKHM)} associated with $k$ is defined as the completion of $\modu_{k,0}$.  
We denote by $\modu_k$ the RKHM associated with $k$.
In the following, we denote the inner product, absolute value, and norm in $\modu_k$ by $\bracket{\cdot,\cdot}_k$, $\vert\cdot\vert_k$, and $\Vert\cdot\Vert_k$\addHK{, respectively}.
\citet{hashimoto21} showed the representer theorem 
in RKHM.
\begin{proposition}[Representer theorem]\label{prop:representation}
Let $\alg$ be a unital $C^*$-algebra.
Let $x_1,\ldots,x_n\in\mcl{X}$ and $y_1,\ldots,y_n\in\alg$.
Let $h:\mcl{X}\times\alg\times\alg\to\alg_+$ be an error function and let $g:\alg_+\to\alg_+$ satisfy $g(a)\lneq_{\alg} g(b)$ for $a\lneq_{\alg} b$.
Assume the module (algebraically) generated by $\{\phi(x_i)\}_{i=1}^n$ is closed.
Then, any $u\in\modu_k$ minimizing $\sum_{i=1}^nh(x_i,y_i,u(x_i))+g(\vert u\vert_{\modu_k})$ admits a representation of the form $\sum_{i=1}^n\phi(x_i)c_i$ for some $c_1,\ldots,c_n\in\alg$.
\end{proposition}

\section{$C^*$-ALGEBRA-VALUED POSITIVE DEFINITE KERNELS}\label{sec:kernels}
To investigate the supervised learning problem in RKHM, we begin by constructing suitable $C^*$-algebra-valued positive definite kernels.
The product structure used in these kernels will be shown to be effective in analyzing image data.
However, the proposed kernels are general, and their application is not limited to image data.

Let $\alg_1$ be a $C^*$-algebra.
By the Gelfand--Naimark theorem (see, for example, \citealt{murphy90}), there exists a Hilbert space $\hil$ such that $\alg_1$ is a subalgebra of the $C^*$-algebra $\alg_2$ of bounded linear operators on $\hil$.
For image data we can set $\alg_1$ and $\alg_2$ as follows.
\begin{example}
Let $p\in\mathbb{N}$, $\alg_1=C^*(\mathbb{Z}/p\mathbb{Z})$, and $\alg_2=\mathbb{C}^{p\times p}$.
Then, $\alg_1$ is a subalgebra of $\alg_2$.
Indeed, by Lemma~\ref{lem:circulant_isomorphic}, $\alg_1\simeq Circ(p)$.
{For example, in image processing, we represent filters by circulant matrices~\citep{bhabatosh11}.
If we regard $\mathbb{Z}/p\mathbb{Z}$ as the space of $p$ pixels, then elements in $C^*(\mathbb{Z}/p\mathbb{Z})$ can be regarded as functions from pixels to intensities.
Thus, we can also regard grayscale and color images with $p$ pixels as elements in $C^*(\mathbb{Z}/p\mathbb{Z})$ and $C^*(\mathbb{Z}/p\mathbb{Z})^3$, respectively.}
Note that $\alg_2$ is noncommutative, although $\alg_1$ is commutative.
\end{example}
We consider the case where the inputs are in $\alg_1^d$ for $d\in\mathbb{N}$ and define {linear, polynomial, and Gaussian} $C^*$-algebra-valued positive definite kernels as follows.
For example, we can consider the case where inputs are $d$ images.
\begin{definition}\label{def:A-valued_kernel}
Let $\mcl{X}\subseteq\alg_1^d$ and $x=[x_1,\ldots,x_d]\in\mcl{X}$.
\begin{enumerate}
\item For $a_{i,1},a_{i,2}\in\alg_2$, the {\em linear kernel} $k:\mcl{X}\times \mcl{X}\to\alg_2$ is defined as $k(x,y)=\sum_{i=1}^d a_{i,1}^*x_i^*a_{i,2}^*a_{i,2}y_ia_{i,1}$.

\item For $q\in\mathbb{N}$ and $a_{i,j}\in\alg_2\ (i=1,\ldots d,j=1,\ldots q+1)$, the {\em polynomial kernel} $k:\mcl{X}\times \mcl{X}\to\alg_2$ is defined as 
\begin{equation*}
k(x,y)=\sum_{i=1}^d\bigg(\prod_{j=1}^{q} a_{i,j}^*x_i^*\bigg)a_{i,q+1}^*a_{i,q+1}\bigg(\prod_{j=1}^{q}y_i a_{i,q+1-j}\bigg).
\end{equation*}
\todo[disable]{equation above exceeds the margin}

\item Let $\Omega$ be a measurable space and $\mu$ is an $\alg_2$-valued positive measure on $\Omega$.\footnote{\addHK{See~\citet[Appendix B]{hashimoto21} for a rigorous definition.}} 
For $a_{i,1},a_{i,2}:\Omega\to\alg_2$, the {\em Gaussian kernel} $k:\mcl{X}\times \mcl{X}\to\alg_2$ is defined as
\begin{align*}
k(x,y)=&\int_{\omega\in\Omega}\mr{e}^{-\sqrt{-1}\sum_{i=1}^da_{i,1}(\omega)^*x_i^*a_{i,2}(\omega)^*}\mr{d}\mu(\omega)\mr{e}^{\sqrt{-1}\sum_{i=1}^da_{i,2}(\omega)y_ia_{i,1}(\omega)}.
\end{align*}
Here, we assume the integral does not diverge.
\end{enumerate}
\end{definition}
\begin{remark}
We can construct new kernels by the composition of functions to the kernels defined in Definition~\ref{def:A-valued_kernel}.
For example, let $\psi_{i,j}:\alg_1\to\alg_2$\; for $i=1,\ldots d$ and $j=1,\ldots,q+1$.
Then, the map defined by replacing $x_i$ and $y_i$ in the polynomial kernel by $\psi_{i,j}(x_i)$ and $\psi_{i,j}(y_i)$
is also an $C^*$-algebra-valued positive definite kernel.
\end{remark}

\begin{figure}[t]
    \centering
    \includegraphics[scale=0.37]{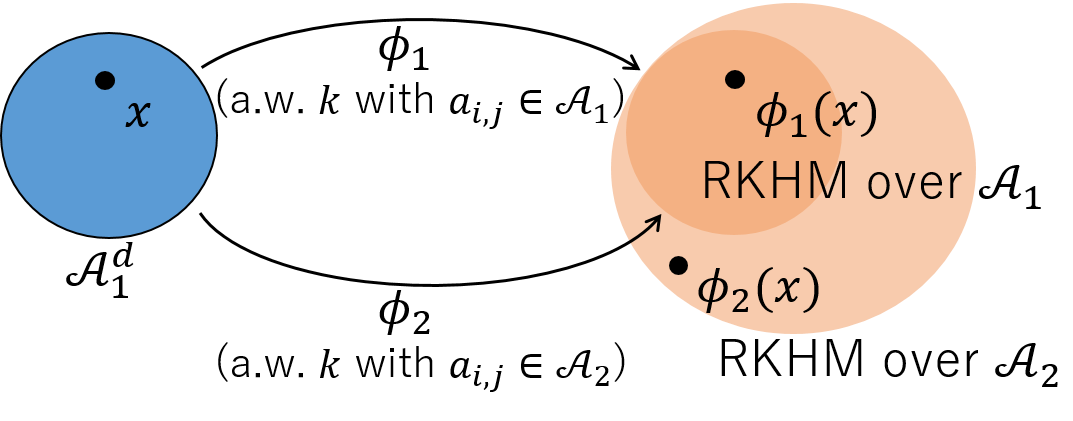}\vspace{-.3cm}
    \caption{Representing samples in RKHM}
    \label{fig:overview}
\end{figure}

If $\alg_1=\alg_2=\mathbb{C}$, then the above kernels are reduced to the standard complex-valued positive definite kernels and the RKHMs associated with them are reduced to RKHSs.
In this case, if $\mcl{X}=\alg^d$, the input space and the RKHS are both Hilbert spaces (Hilbert $\mathbb{C}$-modules).
On the other hand, for RKHMs, if we choose $\alg_1\subsetneq\alg_2$, then the input space $\mcl{X}$ is a Hilbert $\alg_1$-module, but the RKHM is a Hilbert $\alg_2$-module, not $\alg_1$-module.
Applying RKHMs, {we can construct higher dimensional spaces than input spaces but also enlarge the $C^*$-algebras characterizing the RKHMs},
which allows us to represent more functions than RKHSs and make use of the product structure in $\alg_2$.
Figure~\ref{fig:overview} schematically shows the representation of samples in RKHM.
We show an example related to image data below.

\begin{example}\label{ex:FC}
If $\alg_1=C^*(\mathbb{Z}/p\mathbb{Z})$, $\alg_2=\mathbb{C}^{p\times p}$ ($\alg_1\subsetneq\alg_2$), and $a_{i,j}\in\alg_1$, then $a_{i,j}$ in Definition~\ref{def:A-valued_kernel} behaves as convolutional filters.
In fact, by Definition~\ref{def:group_c_algebra}, the multiplication of $a_{i,j}$ and $x_i$ is represented by the convolution.
The convolution of two functions corresponds to the multiplication of each Fourier component of them.
Thus, each Fourier component of $x_i$ does not interact with other Fourier components.
Choosing $a_{i,j}\in\alg_2$ outside $\alg_1$ corresponds to the multiplication of different Fourier components of two functions.
Indeed, let $x\in\alg_1$. Then, by Lemma~\ref{lem:circulant_isomorphic}, $x$ is represented as a circulant matrix and by Lemma~\ref{lem:circulant_decom}, it is decomposed as $x=F\Lambda_xF^*$. 
In this case, $\Lambda_x$ is the diagonal matrix whose $i$th diagonal is 
the $i$th Fourier component (FC) of $x$.
Thus, if $a_{i,j}\in\alg_1$, then we have $xa_{i,j}=F\Lambda_x\Lambda_{a_{i,j}}F^*$ and each Fourier component of $x$ is multiplied by the same Fourier component of $a_{i,j}$.
On the other hand, if $a_{i,j}\in\alg_2\setminus\alg_1$, then $\Lambda_{a_{i,j}}$ is not a diagonal matrix, and the elements of $\Lambda_x\Lambda_{a_{i,j}}$ are composed of the weighted sum of different Fourier components of $x$.
Figure~\ref{fig:overview_FC} summarizes this example.
\end{example}
\paragraph{Comparison with vvRKHS}
From the perspective of vvRKHS, defining kernels 
\addHK{as in} Definition~\ref{def:A-valued_kernel} is difficult since for vvRKHS, the output space is a Hilbert space, and we do not have product structures in it.
Indeed, the inner product in a vvRKHS is described by an action of an operator on a vector.
We can regard the vector as a rank-one operator whose range is the one-dimensional space spanned by the vector.
Thus, the action is regarded as the product of only two operators.
On the other hand, from the perspective of $C^*$-algebra, we can multiply more than two elements in $C^*$-algebra, which allows us to define $C^*$-algebra-valued kernels naturally in the same manner as complex-valued kernels.
See Figure~\ref{fig:overview_kernel} for a schematic explanation.
\begin{figure}[t]
\begin{minipage}{.49\textwidth}
    \centering
    \includegraphics[scale=0.4]{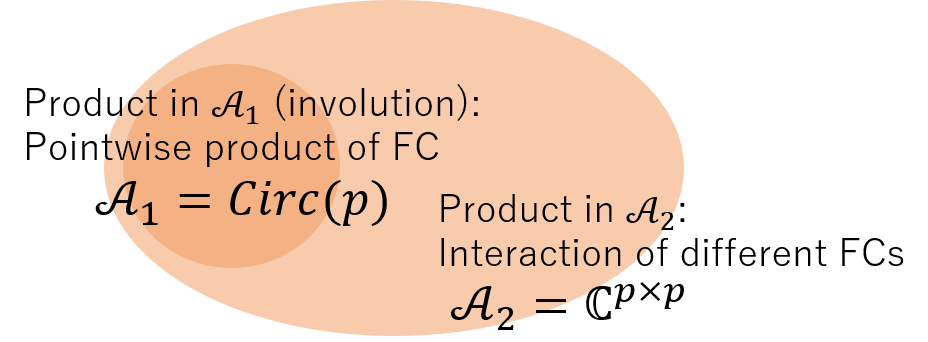}\vspace{-.3cm}
    \caption{Product in $\alg_1$ and $\alg_2$ in Example~\ref{ex:FC}}\smallskip
    \label{fig:overview_FC}
\end{minipage}%
\begin{minipage}{.49\textwidth}
    \centering
    \includegraphics[scale=0.4]{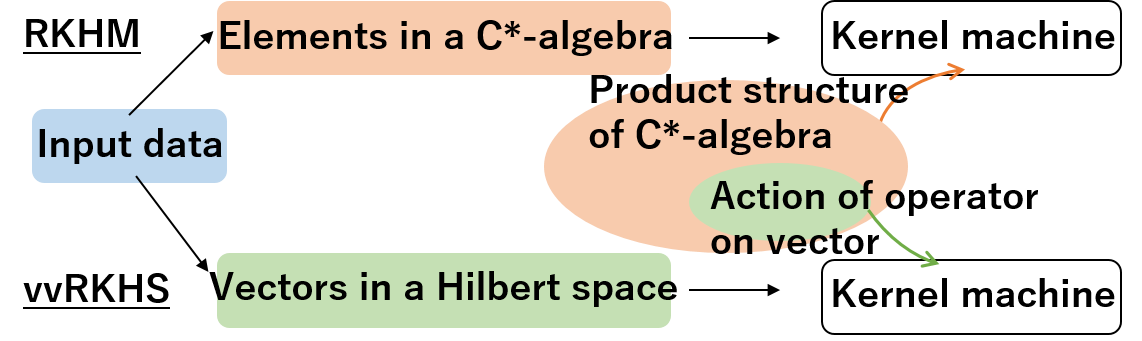}\vspace{-.3cm}
    \caption{Comparison of RKHM with vvRKHS}
    \label{fig:overview_kernel}
\end{minipage}
\end{figure}

\color{black}
\section{SUPERVISED LEARNING IN RKHM}\label{sec:supervised}
We investigate supervised learning in RKHM.
We first formulate the problem \addHK{and derive a learning algorithm}.
Then, we \addHK{characterize} its generalization \addHK{error} and investigate its computational complexity.

We do not assume $\mcl{X}\subseteq\alg_1^d$ in Subsections~\ref{subsec:problem_setting} and \ref{subsec:generalization_bound}.
The input space $\mcl{X}$ can be an arbitrary nonempty set in these sections.
Thus, although we focus on the case of $\mcl{X}\subseteq\alg_1^d$ in this paper, the supervised learning in RKHM is applied to general problems whose output space is a $C^*$-algebra $\alg$.

\subsection{Problem Setting}\label{subsec:problem_setting}
Let $x_1,\ldots,x_n\in\mcl{X}$ be input training samples and $y_1,\ldots,y_n\in\alg$ be output training samples.
Let $k:\mcl{X}\times\mcl{X}\to\alg$ be an $\alg$-valued positive definite kernel, and let $\phi$ and $\modu_k$ be the feature map and RKHM associated with $k$, respectively.
We find a function $f:\mcl{X}\to\alg$ in $\modu_k$ that maps input data to output data.
For this purpose, we consider the following minimization problem:
\begin{equation}
\min_{f\in\modu_k}\bigg(\sum_{i=1}^n\vert f(x_i)-y_i\vert_{\alg}^2+\lambda \vert f\vert_k^2\bigg),\label{eq:supervised}
\end{equation}
where $\lambda\ge 0$ is the regularization parameter.
By the representer theorem (Proposition~\ref{prop:representation}), we find a solution $f$ in the submodule generated by $\{\phi(x_1),\ldots,\phi(x_n)\}$.
As the case of RKHS\addHK{~\citep{scholkopf01_representer}},
representing $f$ as $\sum_{j=1}^n\phi(x_j)c_j\ (c_j\in\alg)$, the problem is reduced to 
\begin{align}
&\min_{c_j\in\alg}\bigg(\sum_{i=1}^n\bigg\vert \sum_{j=1}^nk(x_i,x_j)c_j-y_i\bigg\vert_{{\alg}}^2+\lambda \bigg\vert \sum_{j=1}^n\phi(x_j)c_j\bigg\vert_{k}^2\bigg)\nn\\
&\quad=\min_{c_j\in\alg}(\bc^*\bG^2\bc-\bc^*\bG\by-\by^*\bG\bc+\lambda \bc^*\bG\bc),\label{eq:min_prob}
\end{align}
where $\bG$ is the $\red{\alg^{n\times n}}$-valued Gram matrix whose $(i,j)$-entry is defined as $k(x_i,x_j)\red{\in\alg}$, $\bc=[c_1,\ldots,c_n]^T$, $\by=[y_1,\ldots,y_n]^T$, and $\vert a\vert_{\alg}=(a^*a)^{1/2}$ for $a\in\alg$.
If $\bG+\lambda I$ is invertible, the solution of Problem~\eqref{eq:min_prob} is $\bc=(\bG+\lambda I)^{-1}\by$.

\subsection{Generalization Bound}\label{subsec:generalization_bound}
We derive a generalization bound of the supervised problem in RKHM.
We first define an $\alg$-valued Rademacher complexity.
Let $(\Omega,P)$ be a probability space. 
For a random variable (measurable map) $g:\Omega\to\alg$, we denote by $\mr{E}[g]$ the Bochner integral of $g$, i.e., $\int_{\omega\in\Omega}g(\omega)\mr{d}P(\omega)$.
\begin{definition}
Let $\sigma_1,\ldots,\sigma_n$ be i.i.d and mean zero $\alg$-valued random variables and let $x_1,\ldots,x_n\in\mcl{X}$ be given samples.
Let $\boldsymbol\sigma=\{\sigma_i\}_{i=1}^n$ and $\mathbf{x}=\{x_i\}_{i=1}^n$.
Let $\mcl{F}$ be a class of functions from $\mcl{X}$ to $\alg$.
The {\em $\alg$-valued empirical Rademacher complexity} $\hat{R}(\mcl{F},\boldsymbol\sigma,\mathbf{x})$ is defined as
\begin{equation*}
\hat{R}(\mcl{F},\boldsymbol\sigma,\mathbf{x})=\mr{E}\bigg[\sup_{f\in\mcl{F}}\bigg\vert \frac{1}{n}\sum_{i=1}^nf(x_i)^*\sigma_i\bigg\vert_{\alg}\bigg].
\end{equation*}
\end{definition}
We derive an upper bound of the complexity of a function space related to the RKHM $\modu_k$.
We assume $\alg$ is 
the $C^*$-algebra of bounded linear operators on a Hilbert space.
\begin{proposition}\label{prop:A_valued_complexity}
Let $B>0$ and let $\mcl{F}=\{f\in\modu_k\,\mid\,\Vert f\Vert_k\le B\}$ and let $C=\int_{\Omega}\Vert \sigma_i(\omega)\Vert_{\alg}^2\mr{d}P(\omega)$. 
Then, we have
\begin{equation*}
\hat{R}(\mcl{F},\boldsymbol\sigma,\mathbf{x})\le_{{\alg}} \frac{B\sqrt{C}}{n}\bigg(\sum_{i=1}^n\Vert k(x_i,x_i)\Vert_{\alg}\bigg)^{1/2}I.
\end{equation*}
\end{proposition}
\addHK{To prove Proposition~\ref{prop:A_valued_complexity},} we {first show} the following $\alg$-valued version of Jensen's inequality.
\begin{lemma}\label{lem:jensen}
For a positive $\alg$-valued random variable $c:\Omega\to\alg_{+}$, we have $\mr{E}[c^{1/2}]\le_{\alg}\mr{E}[c]^{1/2}$.
\end{lemma}
{In Example~\ref{ex:FC}, we focused on the case of $\alg=\mathbb{C}^{p\times p}$, which is effective, for example, in analyzing image data.}
In the following, we focus on that case and consider the trace of matrices.
The trace is an appropriate operation for evaluating matrices. 
It is linear and forms the Hilbert--Schmidt inner product.
Let $B>0$ and $E>0$.
We put $\mcl{F}=\{f\in\modu_k\,\mid\,\Vert f\Vert_k\le B,\ f(x)\in\mathbb{R}^{p\times p}\mbox{ for any }x\in\mcl{X}\}$, ${\mcl{G}(\mcl{F})}=\{\mcl{X}\times\mcl{Y}\ni (x,y)\mapsto\vert f(x)-y\vert_{\alg}^2\in\alg\,\mid\,f\in\mcl{F}\}$, and
$\mcl{Y}=\{y\in\mathbb{R}^{p\times p}\,\mid\,\Vert y\Vert_{\alg}\le E\}$.
Let $x_1,\ldots,x_n\in\mcl{X}$ and $y_1,\ldots,y_n\in\mcl{Y}$.
We assume there exists $D>0$ such that for any $x\in\mcl{X}$, $\Vert k(x,x)\Vert_{\alg}\le D$ and let $L=2\sqrt{2}(B\sqrt{D}+E)$.
Using the upper bound of the Rademacher complexity, we derive the following generalization bound.
\begin{proposition}\label{prop:generalization_bound}
Let $\opn{tr}(a)$ be the trace of $a\in\mathbb{C}^{p\times p}$.
For any $g\in \mcl{G}(\mcl{F})$, {any random variable $z:\Omega\to\mcl{X}\times \mcl{Y}$}, and any $\delta\in (0,1)$, with probability $\ge 1-\delta$, we obtain 
\begin{align*}
&\opn{tr}\bigg(\mr{E}[g(z)]-\frac{1}{n}\sum_{i=1}^ng(x_i,y_i)\bigg)
\le 2 \frac{LB\sqrt{D}p}{\sqrt{n}}+3\sqrt{2D}p\sqrt{\frac{\log(2/{\delta})}{n}}.
\end{align*}
\end{proposition}
Note that the same type of bounds is derived for RKHS~\citep[Theorem 3.3]{mohri18} and for vvRKHS~
(\citealp[Corollary 16]{huusari21}, {\citealp[Theorem 3.1]{sindwani13}, \citealp[Theorem 4.1]{sangnier16}}).
Proposition~\ref{prop:generalization_bound} generalizes them to RKHM.

To show Proposition~\ref{prop:generalization_bound}, we first evaluate the Rademacher complexity with respect to the squared loss function $(x,y)\mapsto\vert f(x)-y\vert_{\alg}^2$.
We use Theorem 3 of~\citet{maurer16} to obtain the following bound.
\begin{lemma}\label{prop:error_rademacher}
Let $s_1,\ldots,s_n$ be $\{-1,1\}$-valued Rademacher variables (i.e. independent uniform random variables taking values in $\{-1,1\}$)
and let $\sigma_1,\ldots,\sigma_n$ be i.i.d. $\alg$-valued random variables each of whose element is the Rademacher variable.
Let $\mathbf{s}=\{s_i\}_{i=1}^n$, and $\mathbf{z}=\{(x_i,y_i)\}_{i=1}^n$.
Then, we have
\begin{align*}
&\opn{tr}\hat{R}(\mcl{G}(\mcl{F}),\mathbf{s},\mathbf{z})
\le {L}\opn{tr}\hat{R}(\mcl{F},\boldsymbol\sigma,\mathbf{x}).
\end{align*}
\end{lemma}
Next, we use Theorem 3.3 of \citet{mohri18} to derive an upper bound of the generalization error.
\begin{lemma}\label{prop:generalization_err}
Let $z:\Omega\to\mcl{X}\times \mcl{Y}$ be a random variable and let $g\in\mcl{G}(\mcl{F})$.
Under the same notations and assumptions as Proposition~\ref{prop:error_rademacher}, for any $\delta\in (0,1)$, with probability $\ge 1-\delta$, we have
\begin{align*}
&\opn{tr}\bigg(\mr{E}[g(z)]-\frac{1}{n}\sum_{i=1}^ng(x_i,y_i)\bigg)
\le 2\opn{tr}\hat{R}(\mcl{G}(\mcl{F}),\mathbf{s},\mathbf{z})+3\sqrt{2D}p\sqrt{\frac{\log({2}/{\delta})}{n}}.
\end{align*}
\end{lemma}

\subsection{Computational Complexity}
As mentioned at the beginning of this section, we need to compute $(\bG+\lambda I)^{-1}\by$ for a Gram matrix $\bG\in\alg^{n\times n}$ and a vector $y\in\alg^n$ for solving the minimization problem~\eqref{eq:min_prob}.
When $\alg=\mathbb{C}^{p\times p}$, we have $\alg^{n\times n}=\mathbb{C}^{np\times np}$, and $\mathbf{G}$ is huge if $n$, the number of samples, or $p$, the dimension of $\alg$, is large.
If we construct the $np$ by $np$ matrix explicitly and compute $(\bG+\lambda I)^{-1}\by$ with a direct method such as Gaussian elimination and back substitution (for example, see \citealt{treffethen97}), the computational complexity is $O(n^3p^3)$.
However, if $\mcl{X}=\alg_1^d$, $\alg_1\subsetneq \alg_2$, and parameters in the positive definite kernel have a specific structure, then we can reduce the computational complexity.
For example, {applying the fast Fourier transform, we can compute a multiplication of the DFT matrix $F$ and a vector with $O(p\log p)$~\citep{vanloan92}}.
\todo[disable]{not clear, maybe just remove 'and the following propositions are derived'}
Let $\alg_1=C^*(\mathbb{Z}/p\mathbb{Z})$ and let $\alg_2=\mathbb{C}^{p\times p}$.
Let $k$ be an $\alg_1$ or $\alg_2$-valued positive definite kernel defined in Definition~\ref{def:A-valued_kernel}.
\begin{proposition}\label{prop:comp_complexity_A0}
For $a_{i,j}\in\alg_1$,
the computational complexity for computing $(\bG+\lambda I)^{-1}\by$ by direct methods {for solving linear systems of equations} is $O(np^2\log p+n^3p)$.
\end{proposition}
We can use an iteration method for linear systems, such as the conjugate gradient (CG) method~\citep{hestens52} to reduce the complexity with respect to $n$.
{Note that we need $O(np^2\log p)$ operations after all the iterations.}
\begin{proposition}\label{prop:comp_complexity_A0_CG}
For $a_{i,j}\in\alg_1$,
the computational complexity for $1$ iteration step of CG method is $O(n^2p)$.
\end{proposition}
\begin{proposition}\label{prop:comp_complexity_A0_A}
Let $a_{i,j}\in\alg_2$ whose number of nonzero elements is $O(p\log p)$.
Then, the computational complexity for $1$ iteration step of CG method is $O(n^2p^2\log p)$.
\end{proposition}
\begin{remark}
If we do not use the structure of $\alg_1$, then the computational complexities in Propositions~\ref{prop:comp_complexity_A0}, \ref{prop:comp_complexity_A0_CG}, and \ref{prop:comp_complexity_A0_A} are $O(n^3p^3)$, $O(n^2p^3)$, and $O(n^2p^3)$, respectively.
\end{remark}

%
%
\addHK{In the case of RKHSs, techniques such as the random Fourier feature have been proposed to alleviate the computational cost of kernel methods~\citep{rahimi07}. It could be interesting to inspect how to further reduce the computational complexity of learning in RKHM using random feature approximations for $C^*$-algebra-valued kernels; this is left for future work.}

\section{CONNECTION WITH EXISTING METHODS }\label{sec:existing}
\subsection{Connection with Convolutional Neural Network}\label{subsec:cnn}
Convolutional neural network (CNN) has been one of the most successful methods \addHK{for} analyzing image data~\citep{lecun98,li21}.
We investigate the connection of the supervised learning problem \addHK{in RKHM} 
with CNN.
In this subsection, we set $\mcl{X}\subseteq\red{\alg_1}=C^*(\mathbb{Z}/p\mathbb{Z})$
and $\alg_2=\mathbb{C}^{p\times p}$.
Since the product in $C^*(\mathbb{Z}/p\mathbb{Z})$ is characterized by the convolution, our framework with a specific $\alg_1$-valued positive definite kernel enables us to reconstruct a similar model as the CNN.

We first provide an $\alg_1$-valued positive definite kernel related to the CNN.
\begin{proposition}\label{prop:cnn_A0valued_kernel}
For $a_1,\ldots,a_L,b_1,\ldots,b_L\in\alg_1$ and $\sigma_1,\ldots,\sigma_L:\alg_1\to\alg_1$ each of which has an expansion $\sigma_j(x)=\sum_{l=1}^{\infty}\alpha_{j,l}x^l$ with $\alpha_{j,l}\ge 0$, 
let $\hat{k}:\mcl{X}\times\mcl{X}\to\alg_1$ be defined as 
\begin{align}
\hat{k}(x,y)=&\sigma_L(b_L^*b_L+\sigma_{L-1}(b_{L-1}^*b_{L-1}+\cdots
+\sigma_2(b_2^*b_2+\sigma_1(b_1^*b_1+x^*a_1^*a_1y)a_2^*a_2)\cdots
 a_{L-1}^*a_{L-1})a_L^*a_L).\label{eq:conv_pdk}
\end{align}
Then, $\hat{k}$ is an $\alg_1$-valued positive definite kernel.
\end{proposition}
Using the positive definite kernel~\eqref{eq:conv_pdk}, the solution $f$ of the problem~\eqref{eq:min_prob} is written as 
\begin{align}
f(x)=&\sum_{i=1}^n\sigma_L(b_L^*b_L+\sigma_{L-1}(b_{L-1}^*b_{L-1}+\cdots
+\sigma_2(b_2^*b_2+\sigma_1(b_1^*b_1+x^*a_1^*a_1x_i)a_2^*a_2)\cdots  a_{L-1}^*a_{L-1})a_L^*a_L)c_i,\label{eq:cnn}
\end{align}
for some $c_i\in\alg_1$.
We regard $a_1^*a_1x_i$ and $a_j^*a_j$ for $j=2,\ldots,L$ as convolutional filters, $b_j^*b_j$ for $j=1,\ldots,L$ as biases, and $\sigma_j$ for $j=1,\ldots,L$ as activation functions.
Then, optimizing $a_1,\ldots,a_L,b_1,\ldots,b_L$ simultaneously with $c_i$ corresponds to learning the CNN of the form~\eqref{eq:cnn}.

The following proposition shows that the $C^*$-algebra-valued polynomial kernel defined in Definition~\ref{def:A-valued_kernel} is general enough to represent the $\alg_1$-valued positive definite kernel $\hat{k}$, related to the CNN.
Therefore, by applying $\alg_2$-valued polynomial kernel, not $\alg_1$-valued polynomial kernel, we can go beyond the method with the convolution.
\begin{proposition}\label{prop:conv_poly}
The $\alg_1$-valued positive definite kernel $\hat{k}$ defined as Eq.~\eqref{eq:conv_pdk} is composed of the sum of $\alg_1$-valued polynomial kernels.
\end{proposition}

\subsection{Connection with Convolutional Kernel}
For image data, a ($\mathbb{C}$-valued) positive definite kernel called convolutional kernel is proposed to bridge a gap between kernel methods and neural networks~\citep{marial14,marial16}.
In this subsection, we construct two $C^*$-algebra-valued positive definite kernels that generalize the convolutional kernel.
Similar to the case of the CNN, we will first show that we can reconstruct the convolutional kernel using a $C^*$-algebra-valued positive definite kernel.
Moreover, we will show that our framework gives another generalization of the convolutional kernel.
A generalization of neural networks to $C^*$-algebra-valued networks is proposed~\citep{hashimoto22}.
This generalization allows us to generalize the analysis of the CNNs with kernel methods to that of $C^*$-algebra-valued CNNs.

Let $\Omega$ be a finite subset of $\mathbb{Z}^m$.
For example, $\Omega$ is the space of $m$-dimensional grids.
Let $\red{\tilde{\alg}_1}$ be the space of $\mathbb{C}$-valued maps on $\Omega$ and \red{$\mcl{X}\subseteq\tilde{\alg}_1$}.
The convolutional kernel is defined as follows~{\citep[Definition 2]{marial14}}.
\begin{definition}
Let $\beta,\sigma>0$.
The {\em convolutional kernel} $\tilde{k}:\red{\mcl{X}\times\mcl{X}}\to\mathbb{C}$ is defined as
\begin{align}
\tilde{k}(x,y)=&\sum_{z,z'\in\Omega}\vert x(z)\vert\,\vert y(z')\vert \mr{e}^{-\frac{1}{2\beta^2}\Vert z-z'\Vert^2}
\mr{e}^{-\frac{1}{2\sigma^2}\vert \tilde{x}(z)-\tilde{y}(z')\vert^2}.\label{eq:conv_kernel}
\end{align}
Here, $\Vert\cdot\Vert$ is the standard norm in $\mathbb{C}^m$.
In addition, for $x\in\mcl{X}$, $\tilde{x}(z)=x(z)/\vert x(z)\vert$.
\end{definition}
\red{Let $\Omega=\{z_1,\ldots,z_p\}$, $\alg_1=C^*(\mathbb{Z}/p\mathbb{Z})$, and $\alg_2=\mathbb{C}^{p\times p}$.}
We first construct an $\alg_1$-valued positive definite kernel, which reconstructs the convolutional kernel~\eqref{eq:conv_kernel}.

\begin{proposition}\label{prop:conv_ver1}
Define $\hat{k}:\mcl{X}\times\mcl{X}\to\alg_1$ as
\begin{align}
\hat{k}(x,y)=\int_{\mathbb{R}}\int_{\mathbb{R}^m}c_x(\omega,\eta)^*c_y(\omega,\eta)\;\mr{d}\lambda_{\beta}(\omega)\mr{d}{\lambda}_{\sigma}(\eta),\label{eq:conv_kernel_tilA0}
\end{align}
where $\mr{d}\lambda_{\beta}(\omega)=\beta\mr{e}^{-\frac{\beta^2\omega^2}{2}}\mr{d}\omega$ for $\beta>0$ and
\begin{align*}
c_x(\omega,\eta)=&\opn{circ}\Big(\vert x(z_1)\vert\mr{e}^{\sqrt{-1}\omega\cdot z_1}\mr{e}^{\sqrt{-1}\eta\cdot \tilde{x}(z_1)},\cdots,
\vert x(z_p)\vert\mr{e}^{\sqrt{-1}\omega\cdot z_p}\mr{e}^{\sqrt{-1}\eta\cdot \tilde{x}(z_p)}\Big),
\end{align*}
for $x\in\mcl{X}$, $\omega\in\mathbb{R}^m$, and $\eta\in\mathbb{R}$.
Then, $\hat{k}$ is an $\alg_1$-valued positive definite kernel, and for any $l=1,\ldots,p$, $\tilde{k}$ is written as
\begin{align*}
\tilde{k}(x,y)=\frac{1}{p}\sum_{i,j=1}^p\hat{k}(x,y)_{i,j}=\sum_{j=1}^p\hat{k}(x,y)_{l,j},
\end{align*}
where $\hat{k}(x,y)_{i,j}$ is the $(i,j)$-entry of $\hat{k}(x,y)$.
\end{proposition}
\begin{remark}
Similar to Subsection~\ref{subsec:cnn}, we can generalize $\hat{k}$ by replacing $c_x(\cdot,\cdot)^*c_y(\cdot,\cdot)$ by an $\alg_2$-valued polynomial kernel with respect to $c_x(\cdot,\cdot)$ and $c_y(\cdot,\cdot)$ in Eq.~\eqref{eq:conv_kernel_tilA0}.
\end{remark}

\red{Instead of $\alg_1$-valued, we can also construct an $\tilde{\alg}_1$-valued kernel}, which reconstructs the convolutional kernel~\eqref{eq:conv_kernel}.
\begin{definition}
Let $\beta,\sigma>0$.
Define $\check{k}:\mcl{X}\times\mcl{X}\to\tilde{\alg}_1$ as
\begin{align}
\check{k}(x,y)(w)=\sum_{z,z'\in\Omega}\vert x(\psi(z,w))\vert\,\vert y(\psi(z',w))\vert
\mr{e}^{\frac{-1}{2\beta^2}\Vert \psi(z,w)-\psi(z',w)\Vert^2}
\mr{e}^{\frac{-1}{2\sigma^2}\vert \tilde{x}(\psi(z,w))-\tilde{y}(\psi(z',w))\vert^2}\label{eq:A0_valued_conv}
\end{align}
for $w\in\Omega$.
Here, $\psi:\Omega\times\Omega\to\Omega$ is a map satisfying $\psi(z,0)=z$ for any $z\in\Omega$.
\end{definition}
The $\tilde{\alg}_1$-valued map $\check{k}$ is a generalization of the ($\mathbb{C}$-valued) convolutional kernel $\tilde{k}$ in the following sense, which is directly derived from the definitions of $\check{k}$ and $\tilde{k}$.
\begin{proposition}
For $\tilde{k}$ and $\check{k}$ defined as Eqs.~\eqref{eq:conv_kernel} and \eqref{eq:A0_valued_conv}, respectively, we have
$\check{k}(x,y)(0)=\tilde{k}(x,y)$. 
\end{proposition}

We further generalize the $\tilde{\alg}_1$-valued kernel $\check{k}$ to an $\alg_2$-valued positive definite kernel.
\begin{definition}
Let $\beta,\sigma>0$ and $a_i\in\alg_2$ for $i=1,2,3,4$.
Let $\psi$ be the same map as that in Eq.~\eqref{eq:A0_valued_conv}.
Let $k:\mcl{X}\times\mcl{X}\to\alg_2$ be defined as
\begin{align}
&k(x,y)=\int_{\mathbb{R}}\int_{\mathbb{R}^m}\sum_{z,z'\in\Omega}a_1^*\mathbf{x}(z)a_2^*b(z,\omega)^*a_3^*\tilde{\mathbf{x}}(z,\eta)^*
a_4^*a_4\tilde{\mathbf{y}}(z',\eta)a_3 b(z',\omega)a_2\mathbf{y}(z')a_1\;\mr{d}\lambda_{\beta}(\omega)\mr{d}\lambda_{\sigma}(\eta) \label{eq:A_valued_conv}
\end{align}
for $x,y\in\mcl{X}$.
Here, for $x\in\mcl{X}$,
\begin{align*}
&\mathbf{x}(z)=\opn{diag}(\vert x(\psi(z,z_1)\vert,\ldots,\vert x(\psi(z,z_p))\vert)\in\alg_2,\\
&\tilde{\mathbf{x}}(z,\omega)
=\opn{diag}(\mr{e}^{-\sqrt{-1}\omega\cdot\tilde{x}(\psi(z,z_1))},\ldots,\mr{e}^{-\sqrt{-1}\omega\cdot\tilde{x}(\psi(z,z_p))})\in\alg_2,\\
&b(z,\omega)=\opn{diag}(\mr{e}^{-\sqrt{-1}\omega\cdot \psi(z,z_1)},\ldots,\mr{e}^{-\sqrt{-1}\omega\cdot \psi(z,z_p)})\in\alg_2.
\end{align*}
\end{definition}
\begin{proposition}\label{prop:conv_A}
The $\alg_2$-valued map $k$ defined as Eq.~\eqref{eq:A_valued_conv} is an $\alg_2$-valued positive definite kernel.
\end{proposition}
The following proposition shows $k$ is a generalization of $\check{k}$, which means we finally generalize the ($\mathbb{C}$-valued) convolution kernel $\tilde{k}$ to an $\alg_2$-valued positive definite kernel.
This allows us to generalize the relationship between the CNNs and the convolutional kernel to that of a $C^*$-algebra-valued version of the CNNs and the $C^*$-algebra-valued convolutional kernel $k$.
\begin{proposition}
If $a_i=I$, then the $\alg_2$-valued positive definite kernel $k$ defined as Eq.~\eqref{eq:A_valued_conv} is reduced to the $\alg_1$-valued convolutional kernel $\check{k}$ defined as Eq.~\eqref{eq:A0_valued_conv}.
\end{proposition}

\begin{table}[t]
    \centering
    \def\arraystretch{1.3}
    \begin{tabularx}{0.7\linewidth}{X|c|c}
         &&Mean error  \\
         \hline
         \multirow{3}{\hsize}{vvRKHS, Gaussian ($\tilde{k}(x,y)=\mr{e}^{-c\Vert x-y\Vert^2}$)}& $k=\tilde{k}I$ & $0.640\pm 0.122$\\
         &$k=\tilde{k}T$ & $0.603\pm 0.028$\\
         &Nonsep & $0.650\pm 0.051$\\
         \hline
         \multirow{3}{\hsize}{vvRKHS, Laplacian ($\tilde{k}(x,y)=\mr{e}^{-c\Vert x-y\Vert}$)}& $k=\tilde{k}I$ & $0.538\pm 0.027$\\
         &$k=\tilde{k}T$ & $0.590\pm 0.021$\\
         &Nonsep & $0.650\pm 0.048$\\
         \hline
         \multirow{3}{\hsize}{vvRKHS, Polynomial ($\tilde{k}(x,y)=\sum_{i=1}^3(1-cx\cdot y)^i$)}& $k=\tilde{k}I$ & $0.800\pm 0.032$\\
         &$k=\tilde{k}T$ & $0.539\pm 0.012$\\
         &Nonsep & $0.539\pm 0.012$\\
         \hline
         \multicolumn{2}{>{\hsize=8cm}X|}{RKHM ($k(x,y)=\sum_{i=1}^3R_x^*(I-cQ_x^*)^i(I-cQ_y)^iR_y$)} & $\mathbf{0.343\pm 0.022}$
    \end{tabularx}
    \medskip\\
    $T=\begin{bmatrix}1&1\\ 1&1 \end{bmatrix}$,\qquad
    Nonsep: $k(x_{1},x_{2})_{i,j}=\tilde{k}(x_{1,i},x_{2,j})$
    \caption{Comparison between an RKHM and vvRKHSs~(Mean value $\pm$ standard deviation of $5$ runs)}
    \label{tab:err_syn}
\end{table}

\begin{figure*}[t]
\begin{minipage}{0.6\textwidth}
    \centering
    \subfigure[Regularization parameter $\lambda$]{\includegraphics[scale=0.24]{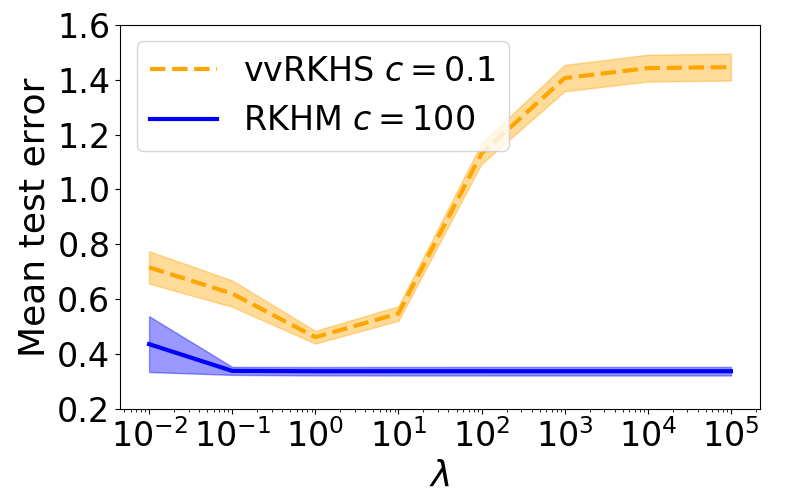}}
    \subfigure[Number of samples $n$]{\includegraphics[scale=0.24]{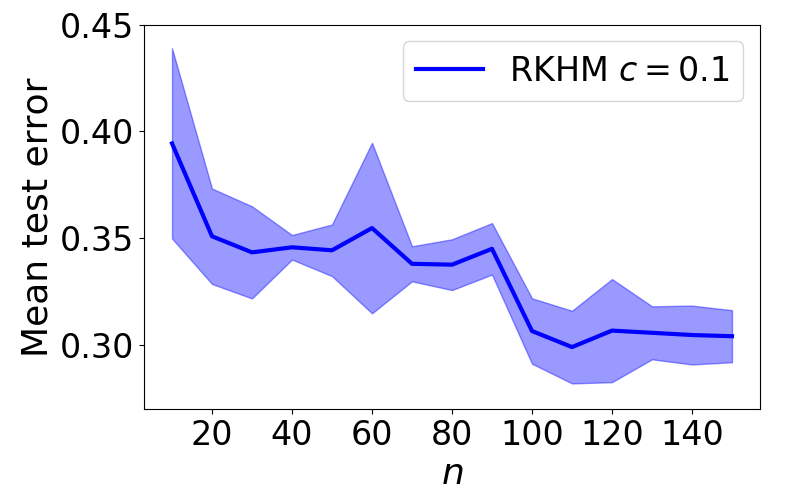}}\vspace{-.3cm}
    \caption{{Mean test error} versus hyperparameters (Mean value $\pm$ standard deviation of $5$ runs).}
    \label{fig:laplacian_rkhm}
\end{minipage}\qquad
\begin{minipage}{0.35\textwidth}
    \centering
\begin{tabular}{cc}
    Original &  
    \includegraphics[scale=0.7]{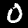}
    \includegraphics[scale=0.7]{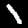}
    \includegraphics[scale=0.7]{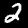}\\
    Input & 
    \includegraphics[scale=0.7]{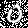}
    \includegraphics[scale=0.7]{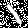}
    \includegraphics[scale=0.7]{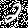}\\
    3-layer CNN &
    \includegraphics[scale=0.7]{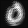}
    \includegraphics[scale=0.7]{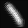}
    \includegraphics[scale=0.7]{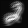}\\
    RKHM $+$ 1-layer CNN&
    \includegraphics[scale=0.7]{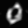}
    \includegraphics[scale=0.7]{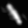}
    \includegraphics[scale=0.7]{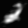}\\
\end{tabular}\vspace{-.15cm}
    \caption{Comparison between RKHM and CNN}
    \label{fig:mnist}
\end{minipage}\vspace{-.2cm}
\end{figure*}

\begin{figure*}[t]
    \centering
    \subfigure[2-layer models]{\includegraphics[scale=0.24]{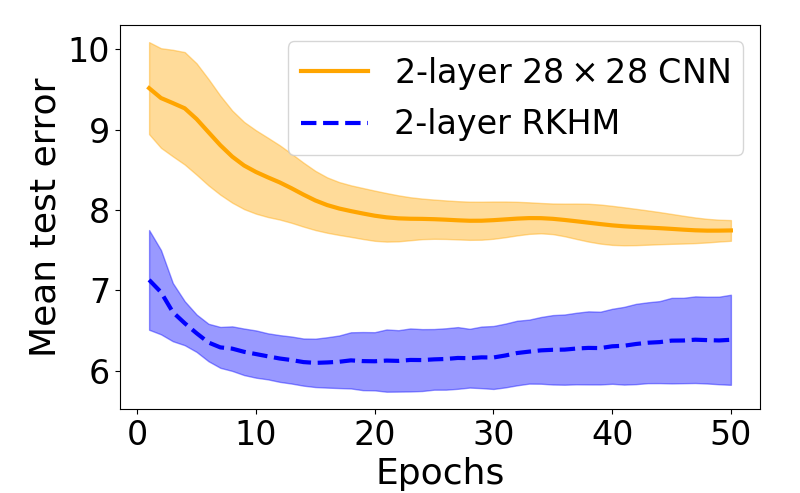}}\qquad
    \subfigure[3-layer models]{\includegraphics[scale=0.24]{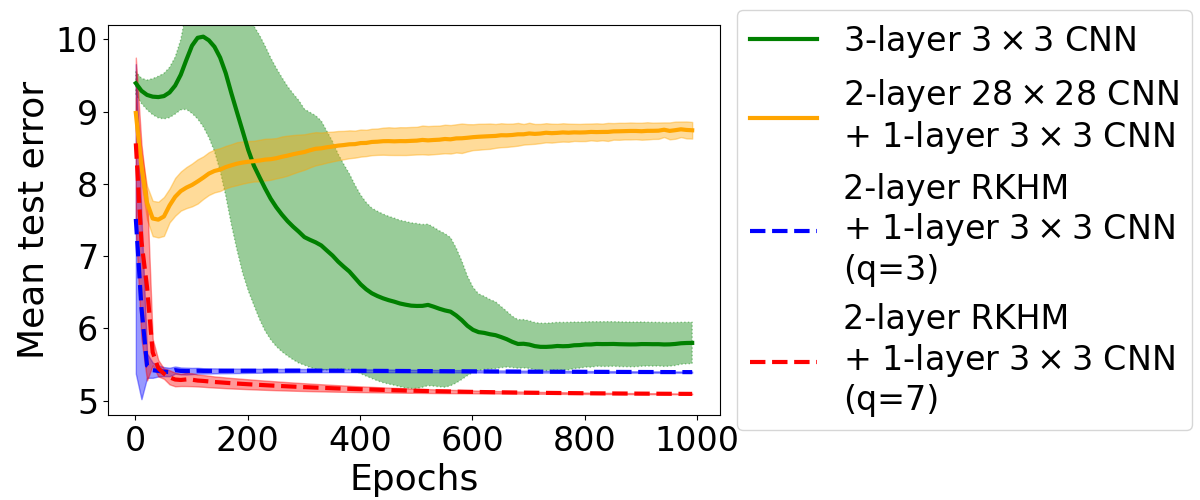}}\vspace{-.3cm}
    \caption{Mean test error versus the number of epochs (Mean value $\pm$ standard deviation of $5$ runs).}
    \label{fig:err_mnist}
\end{figure*}

%

\section{NUMERICAL RESULTS}\label{sec:numerical_results}
\subsection{Experiments with Synthetic Data}
We compared the performances of supervised learning in RKHMs and vvRKHSs. 
We generated $n$ samples $x_1,\ldots,x_{n}$ in $[0,1]^2$ each of whose elements is independently drawn from the uniform distribution on $[0,1]$.
For a generated sample $x_i=[x_{i,1},x_{i,2}]$, we added noise $\xi_i\in\mathbb{R}^2$, each of whose elements is independently drawn from the Gaussian distribution with mean $0$ and standard deviation $0.1$.
We generated the corresponding output sample $y_i$ as $y_i=[\sin(\tilde{x}_{i,1}+\tilde{x}_{i,2}), \sin(\tilde{x}_{i,1}+\tilde{x}_{i,2})+\sin(0.5(\tilde{x}_{i,1}+\tilde{x}_{i,2}))]\in\mathbb{R}^2$, where $\tilde{x}_i=x_i+\xi_i$.
We learned a function $f$ that maps $x_i$ to $y_i$ in different RKHMs and vvRKHSs and different values of the regularization parameter $\lambda$.
To compare the performances, we generated $100$ test input samples $\hat{x}_1,\ldots,\hat{x}_{100}$ in $[0,1]^2$ each of whose elements is independently drawn from the uniform distribution on $[0,1]$.
We also generated $\hat{y}_1,\ldots,\hat{y}_{100}$ given by $\hat{y}_i=[\sin(\hat{x}_{i,1}+\hat{x}_{i,2}), \sin(\hat{x}_{i,1}+\hat{x}_{i,2})+\sin(0.5(\hat{x}_{i,1}+\hat{x}_{i,2}))]$.
We computed the mean error $1/100\sum_{i=1}^{100}\Vert f(\hat{x}_i)-\hat{y}_i\Vert$.
The results for $n=30$ are illustrated in Table~\ref{tab:err_syn} and Figure~\ref{fig:laplacian_rkhm}.
Regarding Table~\ref{tab:err_syn}, we executed a cross-validation grid search to find the best parameters $c$ and $\lambda$, where $c$ is a parameter in the positive definite kernels and $\lambda$ is the regularization parameter.
Regarding Figure~\ref{fig:laplacian_rkhm} (a), we set $c$ as the parameter found by the cross-validation and computed the error for different values of $\lambda$.
{We remark that the mean error for the RKHM becomes large as $\lambda$ becomes large, but because of the scale of the vertical axis, we cannot see the change clearly in the figure.}
We can see \addHK{that} RKHM outperforms vvRKHSs.
We also show the relationship between the mean error and the number of samples in Figure~\ref{fig:laplacian_rkhm} (b).
We can see \addHK{that} the mean error becomes small as the number of samples becomes large.

Regarding the learning in RKHMs, for $i=1,\ldots,n$, we transformed $x_i\in[0,1]^2$ into $\opn{circ}(x_i)\in Circ(2)$.
Then, we set $\alg_1= Circ(2)$ and $\alg_2=\mathbb{C}^{2\times 2}$.
We computed the solution of the minimization problem~\eqref{eq:min_prob} and obtained a function $\hat{f}\in\modu_k$ that maps $\opn{circ}(x_i)$ to $\opn{circ}(y_i)$.
Since the output of the learned function $\hat{f}$ takes its value on $\alg_2$, we computed the mean value of $(1,1)$ and $(2,2)$ entries of $\hat{f}(\hat{x}_i)$ for obtaining the first element of the output vector in $\mathbb{R}^2$ and that of $(1,2)$ and $(2,1)$ entries for the second element.
Regarding the $C^*$-algebra-valued kernel for RKHM, we set $k(x,y)=\sum_{i=1}^3R_x^*(I-cQ_x^*)^i(I-cQ_y)^iR_y$ for $x\in\alg_1$, where $x=Q_xR_x$ is the QR decomposition of $x$.

\subsection{Experiments with MNIST}
We compared our method with CNNs using MNIST~\citep{lecun98}.
For $i=1,\ldots,20$, we generated training samples as follows:
We added noise to each pixel of an original image $y_i$ and generated a noisy image $x_i$.
The noise is drawn from the normal distribution with mean 0 and standard deviation 0.01.
Moreover, each digit (0--9) is contained in the training sample set equally (i.e., the number of samples for each digit is 2).
The image size is $28\times 28$.
We tried to find a function that maps a noisy image to its original image using an RKHM and a CNN.
We represent input and output images $x_i$ and $y_i$ as the circulant matrices $\opn{circ}(x_i)$ and $\opn{circ}(y_i)$ whose first rows are $x_i$ and $y_i$.
Then, we learned the function in the RKHM associated with a polynomial kernel $k(x,y)=(a_3^*\sigma(xa_1+a_2)^*+a_4^*)(\sigma(ya_1+a_2)a_3+a_4)$, where $\sigma(x)=(I-cQ_x)R_x+(I-cQ_x)^qR_x$ for $q\in\mathbb{N}$.
Since $k$ has 4 \red{$\alg_2$-valued} parameters, 
\todo[disable]{ 4 parameters in A1? if we say only 4 parameters, one can think that we have 4 scalars parameters}
it corresponds to a generalization of 2-layer CNN with $28\times 28$ filters (see Subsection~\ref{subsec:cnn}).
Regarding the parameters $a_i$, we used a gradient descent method and optimized them.
{We generated 100 noisy images for test samples in the same manner as the training samples and computed the mean error with respect to them.
We set $q=3$.
For comparison, we also trained a 2-layer CNN with $28\times 28$ filters with the same training samples.
The results are illustrated in Figure~\ref{fig:err_mnist} (a).
We can see \addHK{that} the RKHM outperforms the CNN.}
Moreover, we combined the RKHM with a 1-layer CNN with a $3\times 3$ filter, whose inputs are the outputs of the function learned in the RKHM.
{We also trained a 3-layer CNN with $3\times 3$ filters and a 2-layer CNN with $28\times 28$ filters combined with a 1-layer CNN with a $3\times 3$ filter.}
The results are illustrated in Figures~\ref{fig:mnist} and \ref{fig:err_mnist} (b).
We can see that by replacing convolutional layers with an RKHM, we can achieve better performance.
RKHMs and convolutional layers with $28\times 28$ filters capture global information of images.
According to the results of the CNN with $28\times 28$ filters and the RKHM in Figure~\ref{fig:err_mnist} (b), we can see that the RKHM can capture global information of the images more effectively.
On the other hand, convolutional layers with $3\times 3$ filters capture local information.
Since the 2-layer RKHM combined with a 1-layer CNN with a $3\times 3$ filter outperforms a 3-layer CNN with $3\times 3$ filters, \red{we conclude that the combination of the RKHM and CNN captures the global and local information more effectively.}
\todo[disable]{reformulate the last sentence 'we conclude ...'. not clear}

\section{CONCLUSION}\label{sec:conclusion}
We investigated supervised learning in RKHM and {provided a new twist and insights for kernel methods}.
We constructed $C^*$-algebra-valued kernels {from the perspective of $C^*$-algebra}, which is suitable, for example, for analyzing image data.
We investigated the generalization bound and computational complexity for RKHM learning and showed the connection with existing methods.
RKHMs enable us to construct larger representation spaces than the case of RKHSs and vvRKHSs, and generalize operations such as convolution.
This fact implies the representation power of RKHMs goes beyond that of existing frameworks.


\begin{thebibliography}{38}
\providecommand{\natexlab}[1]{#1}
\providecommand{\url}[1]{\texttt{#1}}
\expandafter\ifx\csname urlstyle\endcsname\relax
  \providecommand{\doi}[1]{doi: #1}\else
  \providecommand{\doi}{doi: \begingroup \urlstyle{rm}\Url}\fi

\bibitem[Aizerman et~al.(1964)Aizerman, Braverman, and
  Rozonoer]{aizerman1964theoretical}
Aizerman, M.~A., Braverman, E.~M., and Rozonoer, L.
\newblock Theoretical foundations of the potential function method in pattern
  recognition learning.
\newblock \emph{Automation and Remote Control}, 25:\penalty0 821--837, 1964.

\bibitem[\'{A}lvarez et~al.(2012)\'{A}lvarez, Rosasco, Lawrence,
  et~al.]{alvarez2012kernels}
\'{A}lvarez, M.~A., Rosasco, L., Lawrence, N.~D., et~al.
\newblock Kernels for vector-valued functions: A review.
\newblock \emph{Foundations and Trends{\textregistered} in Machine Learning},
  4\penalty0 (3):\penalty0 195--266, 2012.

\bibitem[Aronszajn(1950)]{aronszajn1950theory}
Aronszajn, N.
\newblock Theory of reproducing kernels.
\newblock \emph{Transactions of the American mathematical society}, 68\penalty0
  (3):\penalty0 337--404, 1950.

\bibitem[Audiffren \& Kadri(2013)Audiffren and Kadri]{audiffren2013stability}
Audiffren, J. and Kadri, H.
\newblock Stability of multi-task kernel regression algorithms.
\newblock In \emph{Proceedings of the 5th Asian Conference on Machine Learning
  (ACML)}, pp.\  1--16, 2013.

\bibitem[Boser et~al.(1992)Boser, Guyon, and Vapnik]{boser1992training}
Boser, B.~E., Guyon, I.~M., and Vapnik, V.~N.
\newblock A training algorithm for optimal margin classifiers.
\newblock In \emph{Proceedings of the 5th annual workshop on Computational
  learning theory (COLT)}, pp.\  144--152, 1992.

\bibitem[Brouard et~al.(2016)Brouard, Szafranski, and d'Alch{\'e}
  Buc]{brouard2016input}
Brouard, C., Szafranski, M., and d'Alch{\'e} Buc, F.
\newblock Input output kernel regression: Supervised and semi-supervised
  structured output prediction with operator-valued kernels.
\newblock \emph{Journal of Machine Learning Research}, 17:\penalty0 1--48,
  2016.

\bibitem[Caponnetto \& De~Vito(2007)Caponnetto and
  De~Vito]{caponnetto2007optimal}
Caponnetto, A. and De~Vito, E.
\newblock Optimal rates for the regularized least-squares algorithm.
\newblock \emph{Foundations of Computational Mathematics}, 7\penalty0
  (3):\penalty0 331--368, 2007.

\bibitem[Chanda \& Majumder(2011)Chanda and Majumder]{bhabatosh11}
Chanda, B. and Majumder, D.~D.
\newblock \emph{Digital Image Processing and Analysis}.
\newblock PHI Learning, 2nd edition, 2011.

\bibitem[Christmann \& Steinwart(2008)Christmann and
  Steinwart]{christmann2008support}
Christmann, A. and Steinwart, I.
\newblock \emph{Support Vector Machines}.
\newblock Springer, 2008.

\bibitem[Gray(2006)]{gray06}
Gray, R.~M.
\newblock Toeplitz and circulant matrices: A review.
\newblock \emph{Foundations and Trends in Communications and Information
  Theory}, 2\penalty0 (3):\penalty0 155--239, 2006.

\bibitem[Hashimoto et~al.(2021)Hashimoto, Ishikawa, Ikeda, Komura, Katsura, and
  Kawahara]{hashimoto21}
Hashimoto, Y., Ishikawa, I., Ikeda, M., Komura, F., Katsura, T., and Kawahara,
  Y.
\newblock Reproducing kernel {H}ilbert {$C^*$}-module and kernel mean
  embeddings.
\newblock \emph{Journal of Machine Learning Research}, 22\penalty0
  (267):\penalty0 1--56, 2021.

\bibitem[Hashimoto et~al.(2022)Hashimoto, Wang, and Matsui]{hashimoto22}
Hashimoto, Y., Wang, Z., and Matsui, T.
\newblock {$C^*$}-algebra net: a new approach generalizing neural network
  parameters to {$C^*$}-algebra.
\newblock In \emph{Proceedings of the 39th International Conference on Machine
  Learning (ICML)}, 2022.

\bibitem[Heo(2008)]{heo08}
Heo, J.
\newblock Reproducing kernel {H}ilbert {$C^*$}-modules and kernels associated
  with cocycles.
\newblock \emph{Journal of Mathematical Physics}, 49\penalty0 (10):\penalty0
  103507, 2008.

\bibitem[Hestenes \& Stiefel(1952)Hestenes and Stiefel]{hestens52}
Hestenes, M.~R. and Stiefel, E.
\newblock Methods of conjugate gradients for solving linear systems.
\newblock \emph{Journal of Research of the National Bureau of Standards},
  49:\penalty0 409--436, 1952.

\bibitem[Huusari \& Kadri(2021)Huusari and Kadri]{huusari21}
Huusari, R. and Kadri, H.
\newblock Entangled kernels - beyond separability.
\newblock \emph{Journal of Machine Learning Research}, 22\penalty0
  (24):\penalty0 1--40, 2021.

\bibitem[Kadri et~al.(2016)Kadri, Duflos, Preux, Canu, Rakotomamonjy, and
  Audiffren]{kadri16}
Kadri, H., Duflos, E., Preux, P., Canu, S., Rakotomamonjy, A., and Audiffren,
  J.
\newblock Operator-valued kernels for learning from functional response data.
\newblock \emph{Journal of Machine Learning Research}, 17\penalty0
  (20):\penalty0 1--54, 2016.

\bibitem[Kirillov(1976)]{kirillov76}
Kirillov, A.~A.
\newblock \emph{Elements of the Theory of Representations}.
\newblock Springer, 1976.

\bibitem[Laforgue et~al.(2020)Laforgue, Lambert, Brogat-Motte, and
  {d}'Alch{\'e}-Buc]{laforgue20}
Laforgue, P., Lambert, A., Brogat-Motte, L., and {d}'Alch{\'e}-Buc, F.
\newblock Duality in {RKHS}s with infinite dimensional outputs: Application to
  robust losses.
\newblock In \emph{Proceedings of the 37th International Conference on Machine
  Learning (ICML)}, 2020.
  
\bibitem[Lance(1995)]{lance95}
Lance, E.~C.
\newblock \emph{Hilbert {$C^*$}-modules -- a Toolkit for Operator Algebraists}.
\newblock London Mathematical Society Lecture Note Series, vol. 210. Cambridge
  University Press, 1995.

\bibitem[Le{C}un et~al.(1998)Le{C}un, Bottou, Bengio, and Haffner]{lecun98}
Le{C}un, Y., Bottou, L., Bengio, Y., and Haffner, P.
\newblock Gradient-based learning applied to document recognition.
\newblock \emph{Proceedings of the IEEE}, 86\penalty0 (11):\penalty0
  2278--2324, 1998.

\bibitem[Li et~al.(2021)Li, Liu, Yang, Peng, and Zhou]{li21}
Li, Z., Liu, F., Yang, W., Peng, S., and Zhou, J.
\newblock A survey of convolutional neural networks: Analysis, applications,
  and prospects.
\newblock \emph{IEEE Transactions on Neural Networks and Learning Systems},
  2021.

\bibitem[Mairal(2016)]{marial16}
Mairal, J.
\newblock End-to-end kernel learning with supervised convolutional kernel
  networks.
\newblock In \emph{Proceedings of the Advances in Neural Information Processing
  Systems 29 (NIPS)}, 2016.

\bibitem[Mairal et~al.(2014)Mairal, Koniusz, Harchaoui, and Schmid]{marial14}
Mairal, J., Koniusz, P., Harchaoui, Z., and Schmid, C.
\newblock Convolutional kernel networks.
\newblock In \emph{Proceedings of the Advances in Neural Information Processing
  Systems 27 (NIPS)}, 2014.

\bibitem[Manuilov \& Troitsky(2000)Manuilov and Troitsky]{manuilov00}
Manuilov, V.~M. and Troitsky, E.~V.
\newblock {H}ilbert {$C^*$} and {$W^*$}-modules and their morphisms.
\newblock \emph{Journal of Mathematical Sciences}, 98\penalty0 (2):\penalty0
  137--201, 2000.

\bibitem[Maurer(2016)]{maurer16}
Maurer, A.
\newblock A vector-contraction inequality for rademacher complexities.
\newblock In \emph{Proceedings of the 27th International Conference on
  Algorithmic Learning Theory (ALT)}, 2016.

\bibitem[Micchelli \& Pontil(2005)Micchelli and Pontil]{micchelli05}
Micchelli, C.~A. and Pontil, M.
\newblock On learning vector-valued functions.
\newblock \emph{Neural Computation}, 17\penalty0 (1):\penalty0 177--204, 2005.

\bibitem[Minh et~al.(2016)Minh, Bazzani, and Murino]{quang16}
Minh, H.~Q., Bazzani, L., and Murino, V.
\newblock A unifying framework in vector-valued reproducing kernel {H}ilbert
  spaces for manifold regularization and co-regularized multi-view learning.
\newblock \emph{Journal of Machine Learning Research}, 17\penalty0
  (25):\penalty0 1--72, 2016.

\bibitem[Mohri et~al.(2018)Mohri, Rostamizadeh, and Talwalkar]{mohri18}
Mohri, M., Rostamizadeh, A., and Talwalkar, A.
\newblock \emph{Foundations of Machine Learning}.
\newblock MIT press, 2018.

\bibitem[Moslehian(2022)]{moslehian22}
Moslehian, M.~S.
\newblock Vector-valued reproducing kernel {H}ilbert {$C^*$}-modules.
\newblock \emph{Complex Analysis and Operator Theory}, 16\penalty0
  (1):\penalty0 Paper No. 2, 2022.

\bibitem[Murphy(1990)]{murphy90}
Murphy, G.~J.
\newblock \emph{$C^*$-Algebras and Operator Theory}.
\newblock Academic Press, 1990.

\bibitem[Murphy(2012)]{murphy12}
Murphy, K.~P.
\newblock \emph{Machine Learning: A Probabilistic Perspective}.
\newblock The MIT Press, 2012.

\bibitem[Rahimi \& Recht(2007)Rahimi and Recht]{rahimi07}
Rahimi, A. and Recht, B.
\newblock Random features for large-scale kernel machines.
\newblock In \emph{Proceedings of the Advances in Neural Information Processing
  Systems 20 (NIPS)}, 2007.

\bibitem[Sangnier et~al.(2016)Sangnier, Fercoq, and d\textquotesingle
  Alch\'{e}-Buc]{sangnier16}
Sangnier, M., Fercoq, O., and d\textquotesingle Alch\'{e}-Buc, F.
\newblock Joint quantile regression in vector-valued {RKHSs}.
\newblock In \emph{Proceedings of the Advances in Neural Information Processing
  Systems 29 (NIPS)}, 2016.

\bibitem[Sch{\"o}lkopf \& Smola(2002)Sch{\"o}lkopf and
  Smola]{scholkopf2002learning}
Sch{\"o}lkopf, B. and Smola, A.~J.
\newblock \emph{Learning with kernels: support vector machines, regularization,
  optimization, and beyond}.
\newblock MIT press, 2002.

\bibitem[Sch\"{o}lkopf et~al.(2001)Sch\"{o}lkopf, Herbrich, and
  Smola]{scholkopf01_representer}
Sch\"{o}lkopf, B., Herbrich, R., and Smola, A.~J.
\newblock A generalized representer theorem.
\newblock In \emph{Proceedings of the 14th Annual Conference on Computational
  Learning Theory (COLT)}, 2001.

\bibitem[Shawe-Taylor \& Cristianini(2004)Shawe-Taylor and
  Cristianini]{shawe2004kernel}
Shawe-Taylor, J. and Cristianini, N.
\newblock \emph{Kernel Methods for Pattern Analysis}.
\newblock Cambridge university press, 2004.

\bibitem[Sindhwani et~al.(2013)Sindhwani, Minh, and Lozano]{sindwani13}
Sindhwani, V., Minh, H.~Q., and Lozano, A.~C.
\newblock Scalable matrix-valued kernel learning for high-dimensional nonlinear
  multivariate regression and granger causality.
\newblock In \emph{Proceedings of the 29th Conference on Uncertainty in
  Artificial Intelligence (UAI)}, 2013.

\bibitem[Trefethen \& Bau(1997)Trefethen and Bau]{treffethen97}
Trefethen, L.~N. and Bau, D.
\newblock \emph{Numerical Linear Algebra}.
\newblock SIAM, 1997.

\bibitem[{Van Loan}(1992)]{vanloan92}
{Van Loan}, C.
\newblock \emph{Computational Frameworks for the Fast Fourier Transform}.
\newblock SIAM, 1992.

\end{thebibliography}
\subsection*{Acknowledgements}
We would like to thank Dr. Chao Li and Dr. Yicong He for pointing out errors in the experiments in Section 6.2.
Hachem Kadri is partially supported by grant ANR-19-CE23-0011 from the French National Research Agency.
Masahiro Ikeda is partially supported by grant JPMJCR1913 from JST CREST.
\clearpage

\appendix

\renewcommand{\thetable}{\Alph{table}}

\section*{APPENDIX}

\section*{Notation}
The typical notations in this paper are listed in Table~\ref{tab1}.

\section{$C^*$-algebra and Hilbert $C^*$-module}\label{ap:c_algebra}
We provide definitions and a lemma related to $C^*$-algebra and Hilbert $C^*$-module.
\begin{definition}[$C^*$-algebra]~\label{def:c*_algebra}
A set $\alg$ is called a {\em $C^*$-algebra} if it satisfies the following conditions:

\begin{enumerate}
 \item $\alg$ is an algebra over $\mathbb{C}$ and {equipped with} a bijection $(\cdot)^*:\alg\to\alg$ that satisfies the following conditions for $\alpha,\beta\in\mathbb{C}$ and $a,b\in\alg$:

 \leftskip=10pt
 $\bullet$ $(\alpha a+\beta b)^*=\overline{\alpha}a^*+\overline{\beta}b^*$,\qquad
 $\bullet$ $(ab)^*=b^*a^*$,\qquad
 $\bullet$ $(a^*)^*=a$.

 \leftskip=0pt
 \item $\alg$ is a normed space endowed with $\Vert\cdot\Vert_{\alg}$, and for $a,b\in\alg$, $\Vert ab\Vert_{\alg}\le\Vert a\Vert_{\alg}\Vert b\Vert_{\alg}$ holds.
 In addition, $\alg$ is complete with respect to $\Vert\cdot\Vert_{\alg}$.

 \item For $a\in\alg$, $\Vert a^*a\Vert_{\alg}=\Vert a\Vert_{\alg}^2$ holds.
\end{enumerate}
A $C^*$-algebra $\alg$ is called unital if there exists $a\in\alg$ such that $ab=b=ba$ for any $b\in\alg$.
We denote $a$ by $1_{\alg}$.
\end{definition}

\begin{definition}[$C^*$-module]
Let $\modu$ be an abelian group with \red{an} operation $+$.
If $\modu$ {is equipped with} a (right) $\alg$-multiplication, then $\modu$ is called a (right) {\em $C^*$-module} over $\alg$.
\end{definition}

\begin{definition}[$\alg$-valued inner product]
Let $\modu$ be a $C^*$-module over $\alg$.
A {$\mathbb{C}$-linear map with respect to the second variable} $\bracket{\cdot,\cdot}_{\modu}:\modu\times\modu\to\alg$ is called an $\alg$-valued {\em inner product} if it satisfies the following properties for $u,v,{w}\in\modu$ 
and $a,b\in\alg$:
\begin{enumerate}
 \item $\bracket{u,va+{w}b}_{\modu}=\bracket{u,v}_{\modu}a+\bracket{u,{w}}_{\modu}b$,
 \item $\bracket{v,u}_{\modu}=\bracket{u,v}_{\modu}^*$,
 \item $\bracket{u,u}_{\modu}\ge_{\alg} 0$,
 \item If $\bracket{u,u}_{\modu}=0$\red{,} then $u=0$.
\end{enumerate}
\end{definition}
\begin{definition}[$\alg$-valued absolute value and norm]
Let $\modu$ be a $C^*$-module over $\alg$.
For $u\in\modu$, the {\em $\alg$-valued absolute value} $\vert u\vert_{\modu}$ on $\modu$ is defined by the positive element $\vert u \vert_{\modu}$ of $\alg$ such that $\vert u\vert_{\modu}^2=\bracket{u,u}_{\modu}$.   
The nonnegative real-valued norm $\Vert \cdot\Vert_{\modu}$ on $\modu$ is defined by $\Vert u\Vert_{\modu} =\big\Vert\vert u\vert_{\modu}\big\Vert_\alg$. 
\end{definition}

Similar to the case of Hilbert spaces, the following Cauchy--Schwarz inequality for $\alg$-valued inner products is available~\cite[Proposition 1.1]{lance95}.
\begin{lemma}[Cauchy--Schwarz inequality]\label{lem:c-s}
Let $\modu$ be a Hilbert $\alg$-module.
For $u,v\in\modu$, the following inequality holds:
\begin{equation*}
\vert\bracket{u,v}_{\modu}\vert_{\alg}^2\;\le_{\alg}\Vert u\Vert_{\modu}^2\bracket{v,v}_{\modu}.
\end{equation*}
\end{lemma}

\setcounter{table}{0}
\begin{table}[t]
\caption{Notation table}\label{tab1}
\vspace{.3cm}
\renewcommand{\arraystretch}{1.1}
 \begin{tabularx}{\linewidth}{|c|X|}
\hline
$\alg$ & A $C^*$-algebra\\
$\alg_+$ & The subset of $\alg$ composed of all positive elements in $\alg$\\
$\le_{\alg}$ & For $a,b\in\alg$, $a\le_{\alg}b$ means $b-a$ is positive\\
$\lneq_{\alg}$ & For $a,b\in\alg$, $a\lneq_{\alg}b$ means $b-a$ is positive and nonzero\\
$\vert\cdot\vert_{\alg}$ & The $\alg$-valued absolute value in $\alg$ defined as $\vert a\vert_{\alg}=(a^*a)^{1/2}$ for $a\in\alg$.\\
$\mcl{X}$ & An input space\\
$\mcl{Y}$ & An output space\\
$k$ & An $\alg$-valued positive definite kernel\\
$\phi$ & The feature map endowed with $k$\\
$\modu_k$ & The RKHM associated with $k$\\
$\mathbf{G}$ & The $\alg$-valued Gram matrix defined as $\bG_{i,j}=k(x_i,x_j)$ for given samples $x_1,\ldots,x_n\in\mcl{X}$\\
$F$ & The discrete Fourier transform (DFT) matrix, whose $(i,j)$-entry is $\omega^{(i-1)(j-1)}/\sqrt{p}$\\
\hline
 \end{tabularx}
\renewcommand{\arraystretch}{1} 
\end{table}

\section{Proofs}\label{ap:proofs}
We provide the proofs of the propositions and lemmas in the main thesis.
\begin{mythm}[Lemma~\ref{lem:circulant_isomorphic}]
The group $C^*$-algebra $C^*(\mathbb{Z}/p\mathbb{Z})$ is $C^*$-isomorphic to $Circ(p)$.
\end{mythm}
\begin{proof}
Let $f:C^*(\mathbb{Z}/p\mathbb{Z})\to Circ(p)$ be a map defined as $f(x)=\opn{circ}(x(0),\ldots,x(p-1))$.
Then, $f$ is linear and invertible.
In addition, we have
\begin{align*}
f(x)f(y)&=\opn{circ}\bigg(\sum_{z\in\mathbb{Z}/p\mathbb{Z}}x(0-z)y(z),\ldots,\sum_{z\in\mathbb{Z}/p\mathbb{Z}}x(p-1-z)y(z)\bigg)\\
&=\opn{circ}((x\cdot y)(0),\ldots,(x\cdot y)(p-1))=f(x\cdot y),\\
f(x)^*&=\opn{circ}(\overline{x(0)},\overline{x(p-1)},\ldots,\overline{x(1)})=f(x^*),\\
\Vert f(x)\Vert&=\bigg\Vert F\opn{diag}\bigg(\sum_{z\in\mathbb{Z}/p\mathbb{Z}}x(z)\mr{e}^{2\pi \sqrt{-1}z\cdot 0/p},\ldots,\sum_{z\in \mathbb{Z}/p\mathbb{Z}}x(z)\mr{e}^{2\pi \sqrt{-1}z(p-1)/p}\bigg)F^*\bigg\Vert
=\Vert x\Vert,
\end{align*}
where the last formula is derived by Lemma~\ref{lem:circulant_decom}.
Thus, $f$ is a $C^*$-isomorphism.
\end{proof}

In the following, for a probability space $\Omega$ and a random variable (measurable map) $g:\Omega\to\mathbb{C}$, the integral of $g$ is denoted by $\mr{E}[g]$.
\begin{mythm}[Lemma~\ref{lem:jensen}]
For a positive $\alg$-valued random variable $c:\Omega\to\alg_{+}$, we have $\mr{E}[c^{1/2}]\le_{\alg}\mr{E}[c]^{1/2}$.
\end{mythm}
\begin{proof}
For any $\epsilon>0$, let $x_0=\mr{E}[c+\epsilon 1_{\alg}]$, $a=1/2x_0^{-1/2}$, and $b=1/2x_0^{1/2}$.
Then, we have $ax_0+b=x_0^{1/2}$ and for any $x\in\alg_+$, we have
\begin{align*}
(ax+b)^*(ax+b)-x&=\frac{1}{4}xx_0^{-1}x+\frac{1}{4}x+\frac{1}{4}x+\frac{1}{4}x_0-x\\
&=\bigg(\frac{1}{2}x_0^{-1/2}x-\frac{1}{2}x_0^{1/2}\bigg)^*\bigg(\frac{1}{2}x_0^{-1/2}x-\frac{1}{2}x_0^{1/2}\bigg)
=(ax-b)^*(ax-b)\ge_{\alg}0.
\end{align*}
Thus, we have $ax+b=\vert ax+b\vert_{\alg}\ge_{\alg}\vert x^{1/2}\vert_{\alg}=x^{1/2}$.
Therefore, we have
\begin{align*}
\mr{E}\big[(c+\epsilon 1_{\alg})^{1/2}\big]\le_{\alg} \mr{E}[a(c+\epsilon 1_{\alg})+b]
=ax_0+b=x_0^{1/2}=\mr{E}[(c+\epsilon 1_{\alg})]^{1/2}.
\end{align*}
Since $\epsilon>0$ is arbitrary, we have $\mr{E}[c^{1/2}]\le_{\alg}\mr{E}[c]^{1/2}$.
\end{proof}

\begin{mythm}[Proposition \ref{prop:A_valued_complexity}]
Let $B>0$ and let $\mcl{F}=\{f\in\modu_k\,\mid\,\Vert f\Vert_k\le B\}$ and let $C=\int_{\Omega}\Vert \sigma_i(\omega)\Vert_{\alg}^2\mr{d}P(\omega)$. 
Then, we have
\begin{equation*}
\hat{R}(\mcl{F},\boldsymbol\sigma,\mathbf{x})\le_{{\alg}} \frac{B\sqrt{C}}{n}\bigg(\sum_{i=1}^n\Vert k(x_i,x_i)\Vert_{\alg}\bigg)^{1/2}I.
\end{equation*}
\end{mythm}
\begin{proof}
By Lemma~\ref{lem:jensen}, we have
\begin{align*}
&\hat{R}(\mcl{F},\boldsymbol\sigma,\mathbf{x})
=\mr{E}\bigg[\sup_{f\in\mcl{F}}\bigg\vert \frac{1}{n}\sum_{i=1}^nf(x_i)^*\sigma_i\bigg\vert_{\alg}\bigg]
=\frac{1}{n}\mr{E}\bigg[\sup_{f\in\mcl{F}}\bigg\vert \Bbracket{f,\sum_{i=1}^n\phi(x_i)\sigma_i}_k\bigg\vert_{\alg}\bigg]\\
&\quad\le \frac{1}{n}\mr{E}\bigg[\sup_{f\in\mcl{F}}\bigg\vert \sum_{i=1}^n\phi(x_i)\sigma_i\bigg\vert_k\Vert f\Vert_k\bigg]
=\frac{1}{n}\mr{E}\bigg[\bigg\vert \sum_{i=1}^n\phi(x_i)\sigma_i\bigg\vert_kB\bigg]
=\frac{B}{n}\mr{E}\bigg[\bigg(\sum_{i,j=1}^n\sigma_i^*k(x_i,x_j)\sigma_j\bigg)^{1/2}\bigg]\\
&\quad \le \frac{B}{n}\mr{E}\bigg[\sum_{i,j=1}^n\sigma_i^*k(x_i,x_j)\sigma_j\bigg]^{1/2}
=\frac{B}{n}\bigg(\sum_{i=1}^n\mr{E}[\sigma_i^*k(x_i,x_i)\sigma_i]\bigg)^{1/2}
\le \frac{B}{n}\bigg(\sum_{i=1}^n\mr{E}[\Vert \sigma_i^*k(x_i,x_i)\sigma_i\Vert_{\alg}]\bigg)^{1/2}I\\
&\quad\le \frac{B}{n}\mr{E}\big[\Vert \sigma_i\Vert_{\alg}^2\big]^{\red{1/2}}\bigg(\sum_{i=1}^n\Vert k(x_i,x_i)\Vert_{\alg}\bigg)^{1/2}I
=\frac{B\sqrt{C}}{n}\bigg(\sum_{i=1}^n\Vert k(x_i,x_i)\Vert_{\alg}\bigg)^{1/2}I,
\end{align*}
where the third inequality is derived by the Cauchy--Schwartz inequality (Lemma~\ref{lem:c-s}).
\end{proof}

In the following, we put $\alg=\mathbb{C}^{p\times p}$.
The following lemmas are applied for the proofs of Lemmas~\ref{prop:error_rademacher} and \ref{prop:generalization_err}.
\begin{lemma}
Let $a,b\in\mathbb{R}^{p\times p}$ or $a,b\in\alg_+$.
If $a\le_{\alg}b$, then $\opn{tr}(a)\le\opn{tr}(b)$.
\end{lemma}
\begin{proof}
Since $b-a\in\alg_+$, we have
\red{$0\le \opn{tr}(b-a)=\opn{tr}(b)-\opn{tr}(a)$.}
\end{proof}

\begin{lemma}
Let $\mcl{S}$ be a subset of $\alg_+$.
Then, $\opn{tr}(\sup_{s\in\mcl{S}}s)=\sup_{s\in\mcl{S}}\opn{tr}(s)$.
\end{lemma}
\begin{proof}
\underline{($\opn{tr}(\sup_{s\in\mcl{S}}s)\le \sup_{s\in\mcl{S}}\opn{tr}(s)$)}\quad
Let $a=\sup_{s\in\mcl{S}}s$.
Let $\epsilon>0$, and assume for any $t\in\mcl{S}$, we have $(1-\epsilon)a\ge_{\alg} t$.
Then, $(1-\epsilon)a$ is an upper bound of $\mcl{S}$ that satisfies $(1-\epsilon)a\lneq_{\alg}a$, which contradicts the fact that $a$ is the supremum of $\mcl{S}$.
Thus, there exists $t\in\mcl{S}$ such that $(1-\epsilon)a\lneq_{\alg} t$.
Therefore, we have 
\begin{align*}
(1-\epsilon)\opn{tr}(a)\le \opn{tr}(t)\le \sup_{s\in\mcl{S}}\opn{tr}(s).
\end{align*}
Since $\epsilon>0$ is arbitrary, we have $\opn{tr}(a)\le \sup_{s\in\mcl{S}}\opn{tr}(s)$.

\underline{($\opn{tr}(\sup_{s\in\mcl{S}}s)\ge \sup_{s\in\mcl{S}}\opn{tr}(s)$)}\quad
Let $\epsilon>0$.
Then, there exists $t\in\mcl{S}$ such that
\begin{align*}
(1-\epsilon)\sup_{s\in\mcl{S}}\opn{tr}(s)\le \opn{tr}(t)\le \opn{tr}(\sup_{s\in\mcl{S}}s).
\end{align*}
Since $\epsilon>0$ is arbitrary, we have $\opn{tr}(\sup_{s\in\mcl{S}}s)\ge \sup_{s\in\mcl{S}}\opn{tr}(s)$.
\end{proof}

\begin{lemma}
Let $a\in\mathbb{R}^{p\times p}$.
Then, $\opn{tr}(a)\le \opn{tr}(\vert a\vert_{\alg})$.
\end{lemma}
\begin{proof}
Let $\lambda_1,\ldots,\lambda_p$ be eigenvalues of $a$, and let $\kappa_1,\ldots,\kappa_p$ be singular values of $a$.
Then, by Weyl's inequality, we have
\begin{align*}
 \opn{tr}(a)=\sum_{i=1}^p\lambda_i\le \sum_{i=1}^p\vert\lambda_i\vert\le \sum_{i=1}^p\kappa_i=\opn{tr}(\vert a\vert_{\alg}).
\end{align*}
\end{proof}

We now show Lemmas~\ref{prop:error_rademacher} and \ref{prop:generalization_err}.
\begin{mythm}[Lemma~\ref{prop:error_rademacher}]
Let $s_1,\ldots,s_n$ be $\{-1,1\}$-valued Rademacher variables (i.e. independent uniform random variables taking values in $\{-1,1\}$)
and let $\sigma_1,\ldots,\sigma_n$ be i.i.d. $\alg$-valued random variables each of whose element is the Rademacher variable.
Let $\mathbf{s}=\{s_i\}_{i=1}^n$, and $\mathbf{z}=\{(x_i,y_i)\}_{i=1}^n$.
Then, we have
\begin{align*}
&\opn{tr}\hat{R}(\mcl{G}(\mcl{F}),\mathbf{s},\mathbf{z})
\le {L}\opn{tr}\hat{R}(\mcl{F},\boldsymbol\sigma,\mathbf{x}).
\end{align*}
\end{mythm}
\begin{proof}
For $f_1,f_2\in\mcl{F}$, we have
\begin{align*}
\opn{tr}(\vert f_1(x_i)-y_i\vert_{\alg}^2)-\opn{tr}(\vert f_2(x_i)-y_i\vert_{\alg}^2)
&=\opn{tr}((f_1(x_i)-y_i+f_2(x_i)-y_i)^*(f_1(x_i)-y_i-f_2(x_i)+y_i))\\
&=\Vert f_1(x_i)-y_i+f_2(x_i)-y_i\Vert_{\alg}\, \Vert f_1(x_i)-y_i-f_2(x_i)+y_i\Vert_{\mr{HS}},
\end{align*}
where the first equality holds since for $a_1,a_2\in\mathbb{R}^{p\times p}$, $\opn{tr}(a_1^*a_2)=\opn{tr}(a_2^*a_1)$ and $\Vert\cdot\Vert_{\mr{HS}}$ is the Hilbert--Schmidt norm in $\mathbb{C}^{p\times p}$.
In addition, we have
\begin{align*}
\Vert f_1(x_i)-y_i+f_2(x_i)-y_i\Vert_{\alg}
&\le \Vert \bracket{\phi(x_i),f_1+f_2}_k-2y_i\Vert_{\alg}\\
&=\Vert k(x_i,x_i)\Vert^{1/2}\Vert f_1+f_2\Vert_k+2\Vert y_i\Vert_{\alg}
=2(B\sqrt{D}+E)=\frac{L}{\sqrt{2}}.
\end{align*}
Thus, by setting $\psi_i(f)=\opn{tr}(\vert f(x_i)-y_i\vert^2)$, $\phi_i(f)=L/\sqrt{2}f(x_i)$, and $\Vert\cdot\Vert=\Vert\cdot\Vert_{\mr{HS}}$ in Theorem 3 of \citet{maurer16}, we obtain 
\begin{align*}
\opn{tr}\mr{E}\bigg[\sup_{f\in\mcl{F}}\frac{1}{n}\sum_{i=1}^ns_i\vert f(x_i)-y_i\vert_{\alg}^2\bigg]
&=\mr{E}\bigg[\sup_{f\in\mcl{F}}\frac{1}{n}\sum_{i=1}^ns_i\opn{tr}(\vert f(x_i)-y_i\vert_{\alg}^2)\bigg]
\le \sqrt{2}\frac{L}{\sqrt{2}}\mr{E}\bigg[\sup_{f\in\mcl{F}}\frac{1}{n}\sum_{i=1}^n\bracket{\sigma_i,f(x_i)}_{\mr{HS}}\bigg]\\
&\le L\mr{E}\bigg[\sup_{f\in\mcl{F}}\opn{tr}\bigg\vert\frac{1}{n}\sum_{i=1}^nf(x_i)^*\sigma_i\bigg\vert_{\alg}\bigg]
=L\opn{tr}\mr{E}\bigg[\sup_{f\in\mcl{F}}\bigg\vert\frac{1}{n}\sum_{i=1}^nf(x_i)^*\sigma_i\bigg\vert_{\alg}\bigg].
\end{align*}
\end{proof}

\begin{mythm}[Lemma~\ref{prop:generalization_err}]
Let $z:\Omega\to\mcl{X}\times \mcl{Y}$ be a random variable and let $g\in\mcl{G}(\mcl{F})$.
Under the same notations and assumptions as Proposition~\ref{prop:error_rademacher}, for any $\delta\in (0,1)$, with probability $\ge 1-\delta$, we have
\begin{align*}
&\opn{tr}\bigg(\mr{E}[g(z)]-\frac{1}{n}\sum_{i=1}^ng(x_i,y_i)\bigg)
\le 2\opn{tr}\hat{R}(\mcl{G}(\mcl{F}),\mathbf{s},\mathbf{z})+3\sqrt{2D}p\sqrt{\frac{\log({2}/{\delta})}{n}}.
\end{align*}
\end{mythm}
\begin{proof}
For a random variable $S=(z_1,\ldots,z_n):\Omega\to (\mcl{X}\times\mcl{Y})^n$, let $\Phi(S)=\opn{tr}(\sup_{g\in\mcl{G}(\mcl{F})}(\mr{E}[g(z)]-1/n\sum_{i=1}^ng(z_i)))$.
For $i=1,\ldots,n$, let $S_i=(z_1',\ldots,z_n')$, where $z_j=z_j'$ for $j\neq i$ and $z_i\neq z_i'$.
Then, we have
\begin{align*}
\Phi(S)-\Phi(S_i)&\le \opn{tr}\bigg(\sup_{g\in\mcl{G}(\mcl{F})}\bigg(\mr{E}[g(z)]-\frac{1}{n}\sum_{j=1}^ng(z_j)\bigg)-\sup_{g\in\mcl{G}(\mcl{F})}\bigg(\mr{E}[g(z)]-\frac{1}{n}\sum_{j=1}^ng(z_j')\bigg)\bigg)\\
&\le \frac{1}{n}\opn{tr}\bigg(\sup_{g\in\mcl{G}(\mcl{F})}\bigg(\sum_{j=1}^ng(z_j)-\sum_{j=1}^ng(z_j')\bigg)\bigg)
=\frac{1}{n}\opn{tr}\bigg(\sup_{g\in\mcl{G}(\mcl{F})}(g(z_i)-g(z_i'))\bigg)\\
&\le \frac{p}{n}\sup_{g\in\mcl{G}(\mcl{F})}\Vert g(z_i)-g(z_i')\Vert_{\alg}
\le \frac{2\sqrt{D}p}{n}.
\end{align*}
By McDiarmid’s inequality, for any $\delta\in (0,1)$, with probability $\ge 1-\delta/2$, we have
\begin{equation*}
\Phi(S)-\mr{E}[\Phi(S)]\le \sqrt{\frac{1}{2}\sum_{i=1}^n\bigg(\frac{2\sqrt{D}p}{n}\bigg)^2\log\frac{2}{\delta}}
=\sqrt{2D}p\sqrt{\frac{\log\frac{2}{\delta}}{n}}
\end{equation*}
Thus, for any $g\in\mcl{G}(\mcl{F})$, we have
\begin{align*}
\opn{tr}\bigg(\mr{E}[g(z)]-\frac{1}{n}\sum_{i=1}^ng(z_i)\bigg)
\le \Phi(S)
\le \mr{E}[\Phi(S)]+\sqrt{2D}p\sqrt{\frac{\log\frac{2}{\delta}}{n}}.
\end{align*}
For the remaining part, the proof is the same as that of Theorem 3.3 of \citet{mohri18}.
Since 
\begin{equation*}
\Phi(S)=\opn{tr}\bigg(\sup_{g\in\mcl{G}(\mcl{F})}\bigg(\mr{E}[g(z)]-\frac{1}{n}\sum_{i=1}^ng(z_i)\bigg)\bigg)
=\sup_{g\in\mcl{G}(\mcl{F})}\bigg(\mr{E}[\opn{tr}(g(z))]-\frac{1}{n}\sum_{i=1}^n\opn{tr}(g(z_i))\bigg),
\end{equation*}
we replace $g$ in the proof of Theorem 3.3 in \red{\citet{mohri18}} by $z\mapsto\opn{tr}(g(z))$ in our case and derive
\begin{align*}
\opn{tr}\bigg(\mr{E}[g(z)]-\frac{1}{n}\sum_{i=1}^ng(x_i,y_i)\bigg)
&\le 2 \mr{E}\bigg[\sup_{f\in\mcl{F}}\frac{1}{n}\sum_{i=1}^ns_i\opn{tr}\vert f(x_i)-y_i\vert_{\alg}^2\bigg]+3\sqrt{2D}p\sqrt{\frac{\log\frac{2}{\delta}}{n}}\\
&\le 2\opn{tr}\hat{R}(\mcl{G}(\mcl{F}),\mathbf{s},\mathbf{z})+3\sqrt{2D}p\sqrt{\frac{\log\frac{2}{\delta}}{n}}.
\end{align*}
\end{proof}

\begin{mythm}[Proposition~\ref{prop:comp_complexity_A0}]
For $a_{i,j}\in\alg_1$,
the computational complexity for computing $(\bG+\lambda I)^{-1}\by$ by direct methods {for solving linear systems of equations} is $O(np^2\log p+n^3p)$.
\end{mythm}
\begin{proof}
Since all the elements of $\bG$ and $\by$ are in $\alg_1$, we have
\begin{equation*}
(\bG+\lambda I)^{-1}\by=(\bF\boldsymbol\Lambda_{\bG+\lambda I}^{-1}\bF^*)\bF\boldsymbol\Lambda_{\by}F^*
=\bF\boldsymbol\Lambda_{\bG+\lambda I}^{-1}\boldsymbol\Lambda_{\by}F^*,
\end{equation*}
where $\mathbf{F}$ is the $\mathbb{C}^{p\times p}$-valued $n\times n$ diagonal matrix whose diagonal elements are all $F$.
In addition, $\boldsymbol\Lambda_{\bG+\lambda I}$ is the $\alg_1$-valued $n\times n$ whose $(i,j)$-entry is $\Lambda_{k(x_i,x_j)}$, and $\boldsymbol\Lambda_{\by}$ is the vector in $\alg_1^n$ whose $i$th element is $\Lambda_{y_i}$.
If we use the \red{fast} Fourier transformation, then the computational complexity of computing $Fy$ for $y\in\mathbb{C}^{p\times p}$ is $O(p^2\log p)$.
Moreover, since the computational complexity of multiplication $\Lambda_{x}\Lambda_{y}$ for $x,y\in\alg_1$ is $O(p)$, using Gaussian elimination and back substitution, the computational complexity of computing $\boldsymbol\Lambda_{\bG+\lambda I}^{-1}\boldsymbol\Lambda_{\by}$ is $O(n^3p)$.
As a result, the total computational complexity is $O(np^2\log p+n^3p)$.
\end{proof}

\begin{mythm}[Proposition~\ref{prop:comp_complexity_A0_A}]
Let $a_{i,j}\in\alg_2$ whose number of nonzero elements is $O(p\log p)$.
Then, the computational complexity for $1$ iteration step of CG method is $O(n^2p^2\log p)$.
\end{mythm}
\begin{proof}
The computational complexity for computing $1$ iteration step of CG method is equal to that of computing $(\bG+\lambda I)\mathbf{b}$ for $\mathbf{b}\in\alg_2^n$.
For $b\in\alg_2$, the computational complexity of computing $k(x_i,x_j)b$ is $O(p^2\log p)$ since those of computing $a_{i,j}b$ and $x_ib$ are both $O(p^2\log p)$. (For $x_ib$, we use fast Fourier transformation.)
Therefore, the computational complexity of computing $(\bG+\lambda I)\mathbf{b}$ is $O(n^2p^2\log p)$.
\end{proof}

\setcounter{equation}{2}
\begin{mythm}[Proposition~\ref{prop:cnn_A0valued_kernel}]
For $a_1,\ldots,a_L,b_1,\ldots,b_L\in\alg_1$ and $\sigma_1,\ldots,\sigma_L:\alg_1\to\alg_1$ each of which has an expansion $\sigma_j(x)=\sum_{l=1}^{\infty}\alpha_{j,l}x^l$ with $\alpha_{j,l}\ge 0$, 
let $\hat{k}:\mcl{X}\times\mcl{X}\to\alg_1$ be defined as 
\begin{align}
\hat{k}(x,y)=&\sigma_L(b_L^*b_L+\sigma_{L-1}(b_{L-1}^*b_{L-1}+\cdots
+\sigma_2(b_2^*b_2+\sigma_1(b_1^*b_1+x^*a_1^*a_1y)a_2^*a_2)\cdots
\times a_{L-1}^*a_{L-1})a_L^*a_L).
\end{align}
Then, $\hat{k}$ is an $\alg_1$-valued positive definite kernel.
\end{mythm}
\begin{proof}
Let $l:\mcl{X}\times\mcl{X}\to\alg_1$ be an $\alg_1$-valued positive definite kernel and $\sigma:\alg_1\to\alg_1$ be a map that has an expansion $\sigma(x)=\sum_{j=1}^{\infty}\alpha_jx^j$ with $\alpha_j\ge 0$.
Then, $\sigma\circ l$ is also an $\alg_1$-valued positive definite kernel.
Indeed, for $d_1,\ldots,d_n\in\alg_1$ and $x_1,\ldots,x_n\in\mcl{X}$, we have
\begin{align*}
\sum_{i,j=1}^nd_i^*\sigma(l(x_i,x_j))d_j
=\sum_{i,j=1}^{n}\sum_{s=1}^{\infty}\alpha_s d_i^*l(x_i,x_j)^sd_j\ge_{\alg_1}0.
\end{align*}
Since $(x,y)\mapsto b_1^*b_1+x^*a_1^*a_1y$ is an $\alg_1$-valued positive definite kernel, $(x,y)\mapsto \sigma_1(b_1^*b_1+x^*a_1^*a_1y)$ is also an $\alg_1$-valued positive definite kernel.
Moreover, since $\sigma_1(b_1^*b_1+x^*a_1^*a_1y)$ and $b_2$ are in $\alg_1$, $(x,y)\mapsto b_2^*b_2+\sigma_1(b_1^*b_1+x^*a_1^*a_1y)$ is also an $\alg_1$-valued positive definite kernel.
We iteratively apply the above result and obtain the positive definiteness of $\hat{k}$.
\end{proof}

\begin{mythm}[Proposition~\ref{prop:conv_poly}]
The $\alg_1$-valued positive definite kernel $\hat{k}$ defined as Eq.~\eqref{eq:conv_pdk} is composed of the sum of $\alg_1$-valued polynomial kernels.
\end{mythm}
\begin{proof}
Since $(x,y)\mapsto b_1^*b_1+x^*a_1^*a_1y$ is an $\alg_1$-valued polynomial kernel and $\sigma:\alg_1\to\alg_1$ is a map that has an expansion $\sigma(x)=\sum_{j=1}^{\infty}\alpha_jx^j$, $\hat{k}$ is composed of the sum of $\alg_1$-valued polynomial kernels.
\end{proof}

\setcounter{equation}{5}
\begin{mythm}[Proposition~\ref{prop:conv_ver1}]
Define $\hat{k}:\mcl{X}\times\mcl{X}\to\alg_1$ as
\begin{align}
\hat{k}(x,y)=\int_{\mathbb{R}}\int_{\mathbb{R}^m}c_x(\omega,\eta)^*c_y(\omega,\eta)\;\mr{d}\lambda_{\beta}(\omega)\mr{d}{\lambda}_{\sigma}(\eta),
\end{align}
where $\mr{d}\lambda_{\beta}(\omega)=\beta\mr{e}^{-\frac{\beta^2\omega^2}{2}}\mr{d}\omega$ for $\beta>0$ and
\begin{align*}
c_x(\omega,\eta)=&\opn{circ}\Big(\vert x(z_1)\vert\mr{e}^{\sqrt{-1}\omega\cdot z_1}\mr{e}^{\sqrt{-1}\eta\cdot \tilde{x}(z_1)},\cdots,
\vert x(z_p)\vert\mr{e}^{\sqrt{-1}\omega\cdot z_p}\mr{e}^{\sqrt{-1}\eta\cdot \tilde{x}(z_p)}\Big),
\end{align*}
for $x\in\mcl{X}$, $\omega\in\mathbb{R}^m$, and $\eta\in\mathbb{R}$.
Then, $\hat{k}$ is an $\alg_1$-valued positive definite kernel, and for any $l=1,\ldots,p$, $\tilde{k}$ is written as
\begin{align*}
\tilde{k}(x,y)=\frac{1}{p}\sum_{i,j=1}^p\hat{k}(x,y)_{i,j}=\sum_{j=1}^p\hat{k}(x,y)_{l,j},
\end{align*}
where $\hat{k}(x,y)_{i,j}$ is the $(i,j)$-entry of $\hat{k}(x,y)$.
\end{mythm}
\begin{proof}
The positive definiteness of $\hat{k}$ is trivial.
As for the relationship between $\tilde{k}$ and $\hat{k}$, we have
\begin{align*}
\hat{k}(x,y)_{i,j}&=\int_{\mathbb{R}}\int_{\mathbb{R}^m}\sum_{l=1}^p\vert x(z_{p-i+2+l})\vert\mr{e}^{-\sqrt{-1}\omega\cdot z_{p-i+2+l}}\mr{e}^{-\sqrt{-1}\eta\cdot \tilde{x}(z_{p-i+2+l})}\\ &\qquad\qquad\times \vert y(z_{p-j+2+l})\vert\mr{e}^{\sqrt{-1}\omega\cdot z_{p-j+2+l}}\mr{e}^{\sqrt{-1}\eta\cdot \tilde{y}(z_{p-j+2+l})}\mr{d}\lambda_{\beta}(\omega)\mr{d}{\lambda}_{\sigma}(\eta)\\
&=\sum_{l=1}^p\vert x(z_{p-i+2+l})\vert\;\vert y(z_{p-j+2+l})\vert \mr{e}^{-\frac{1}{2\sigma^2}\vert\tilde{x}(z_{p-i+2+l})-\tilde{y}(z_{p-j+2+l})\vert^2}\mr{e}^{-\frac{1}{2\beta^2}\Vert z_{p-i+2+l}-z_{p-j+2+l}\Vert^2}.
\end{align*}
Thus, we have
\begin{align*}
\sum_{i=1}^p\hat{k}(x,y)_{i,j}
&=\sum_{i,l=1}^p\vert x(z_{l+j-i})\vert\;\vert y(z_{l})\vert \mr{e}^{-\frac{1}{2\beta^2}\vert \tilde{x}(z_{l+j-i})-\tilde{y}(z_{l})\vert^2}\mr{e}^{-\frac{1}{2\sigma^2}\Vert z_{l+j-i}-z_{l}\Vert^2}=\tilde{k}(x,y).
\end{align*}
\end{proof}

\begin{mythm}[Proposition~\ref{prop:conv_A}]
The $\alg_2$-valued map $k$ defined as Eq.~\eqref{eq:A_valued_conv} is an $\alg_2$-valued positive definite kernel.
\end{mythm}
\begin{proof}
For $n\in\mathbb{N}$, $c_1,\ldots,c_n\in\alg_2$, and $x_1,\ldots,x_n\in\alg_1^d$, we have
\begin{align*}
&\sum_{i,j=1}^nc_i^*k(x_i,x_j)c_j\\
&=\int_{\mathbb{R}}\int_{\mathbb{R}^m}\sum_{i=1}^n\sum_{z\in\Omega}c_i^*a_1^*\mathbf{x}_i(z)a_2^*b(z,\omega)^*a_3^*\tilde{\mathbf{x}}_i(z,\eta)^*a_4^*\sum_{j=1}^n\sum_{z'\in\Omega}a_4\tilde{\mathbf{x}}_j(z',\eta)a_3 b(z',\omega)a_2\mathbf{x}_j(z')a_1c_j\;\mr{d}\lambda_{\beta}(\omega)\mr{d}\lambda_{\sigma}(\eta),
\end{align*}
which is positive semi-definite.
\end{proof}


%
%

\end{document}